\documentclass[10pt]{article}
\usepackage[accepted]{tmlr}

\usepackage{amsmath,amssymb,amsfonts}
\usepackage{amsthm}
\usepackage{mathrsfs}
\usepackage{xcolor}
\usepackage{textcomp}
\usepackage{booktabs}
\usepackage{multirow}
\usepackage{threeparttable}
\usepackage{graphicx}
\usepackage{subcaption}

\usepackage{algorithm}
\usepackage{algorithmicx}
\usepackage{algpseudocode}
\usepackage{listings}
\usepackage{placeins}

\usepackage{mcr}
\usepackage{bm}

\usepackage{url}
\usepackage{hyperref}
\usepackage{soul}
\usepackage[normalem]{ulem}
\newtheorem{remark}{Remark}

\newtheorem{theorem}{Theorem}

\newtheorem{lemma}{Lemma}
\newtheorem{assumption}{Assumption}
\newtheorem{corollary}{Corollary}

\newcommand{\vx}{\boldsymbol{x}}
\newcommand{\vX}{\boldsymbol{X}}

\newcommand{\vy}{\boldsymbol{y}}
\newcommand{\vc}{\boldsymbol{c}}
\newcommand{\vs}{\boldsymbol{s}}
\newcommand{\veps}{\boldsymbol{\varepsilon}}

\newcommand{\veta}{\boldsymbol{\eta}}

\newcommand{\vbeta}{\boldsymbol{\beta}}
\newcommand{\va}{\boldsymbol{a}}
\newcommand{\vb}{\boldsymbol{b}}
\newcommand{\vh}{\boldsymbol{h}}
\newcommand{\vu}{\boldsymbol{u}}

\newcommand{\mG}{G}
\newcommand{\mL}{L}
\newcommand{\mI}{I}
\newcommand{\mSigma}{\Sigma}

\title{Post-Anomaly Detection Inference for Deep SVDD}

\author{\name Cao Le Cong Thanh \email 23521437@gm.uit.edu.vn \\
      \addr University of Information Technology, Ho Chi Minh City, Vietnam\\
      Vietnam National University, Ho Chi Minh City, Vietnam
      \AND
      \name Dang Quang Vinh \email 23521786@gm.uit.edu.vn \\
      University of Information Technology, Ho Chi Minh City, Vietnam\\
      Vietnam National University, Ho Chi Minh City, Vietnam
      \AND
      \name Vo Nguyen Le Duy\thanks{Corresponding author} \email duyvnl@uit.edu.vn\\
      University of Information Technology, Ho Chi Minh City, Vietnam\\
      Vietnam National University, Ho Chi Minh City, Vietnam}

\def\month{09}  
\def\year{2026} 
\def\openreview{\url{https://openreview.net/forum?id=f8XTHjxBig}} 

\begin{document}

\maketitle

\begin{abstract}
Deep Support Vector Data Description (Deep SVDD) has become a prominent framework for unsupervised anomaly detection by learning latent representations that compactly characterize normal data around a center. Despite its empirical success, anomaly decisions produced by Deep SVDD are typically made solely based on anomaly scores without rigorous statistical guarantees, thereby limiting their reliability in safety-critical and high-stakes applications where false positives must be strictly controlled.
In this paper, we propose \emph{PADI} (Post-Anomaly Detection Inference), a novel framework that equips a trained and frozen Deep SVDD detector with statistically valid inference by leveraging the Selective Inference framework. Specifically, PADI performs inference conditional on the event that a test instance is identified as anomalous by Deep SVDD, thereby enabling rigorous statistical assessment of anomaly decisions. Based on this formulation, we derive valid selective $p$-values that quantify the statistical significance of the detected anomaly. Using these $p$-values, we theoretically establish control of the false positive rate (FPR) at a user-specified significance level $\alpha$ (e.g., $\alpha=0.05$).
Furthermore, we extend the proposed framework to Deep Semi-Supervised Anomaly Detection (Deep SAD), providing a principled approach for statistically reliable inference in semi-supervised anomaly detection settings.
Extensive experiments on both synthetic and real-world benchmark datasets robustly support the theoretical findings. The results demonstrate that PADI consistently achieves proper FPR control while attaining superior true positive rates compared with existing approaches.
\end{abstract}

\section{Introduction}\label{sec1}

Anomaly detection (AD) plays a fundamental role in modern machine learning, with applications spanning cybersecurity, healthcare, bioinformatics, industrial monitoring, finance, and autonomous systems \citep{ahmed2016survey, litjens2017survey, zong2018deep}. Among existing approaches, Support Vector Data Description (SVDD)  \citep{tax2004support} has emerged as one of the most influential frameworks for one-class classification and unsupervised AD. The central idea of SVDD is to learn a compact description of normal data by enclosing the normal samples within a minimal hypersphere in a feature space, such that samples lying far from the learned description are identified as anomalies. Owing to its conceptual simplicity and strong empirical effectiveness, SVDD and its deep variants, particularly Deep SVDD \citep{ruff2018deep}, have been successfully applied to a wide range of problems \citep{yi2020patch, gamper2020meta, you2021anomaly, zhang2022deep, kou2022robust}.

Despite the empirical success, Deep SVDD suffers from a critical limitation: anomaly decisions are made based on anomaly scores or heuristic thresholds without rigorous statistical guarantees. In practice, this means that the false positive rate (FPR) cannot be reliably controlled. Such a limitation becomes particularly problematic in high-stakes applications where false positives may lead to severe consequences. For example, in bioinformatics and medical diagnosis, incorrectly flagging healthy patients or normal biological samples as anomalous may trigger unnecessary follow-up procedures, expensive laboratory analyses, or inappropriate clinical interventions. 
Similarly, in cybersecurity, excessive false positives can overwhelm security analysts, leading to alert fatigue and potentially causing truly malicious activities to be overlooked. 
These challenges highlight the importance of developing a statistically reliable method capable of quantifying the uncertainty of anomaly decisions and rigorously controlling the FPR.

A natural approach to addressing this problem is to formulate anomaly assessment as a statistical hypothesis testing problem. Specifically, given a test instance detected as anomalous by an SVDD-based detector, one aims to quantify its statistical significance. However, constructing valid statistical inference in this setting is highly challenging due to the well-known issue of \emph{double dipping} \citep{kriegeskorte2009circular} or \emph{selection bias}. The same data are used both to identify anomalous instances and to conduct statistical inference, resulting in invalid classical $p$-values and inflated FPR. Consequently, conventional (naive) statistical testing procedures fail to provide reliable FPR control for the Deep SVDD-based AD result.

To overcome this challenge, we leverage the framework of Selective Inference (SI) \citep{lee2016exact}, which enables valid statistical inference after a data-driven selection procedure. Building upon this principle, we propose \emph{PADI} (Post-Anomaly Detection Inference), a novel framework that equips a trained and frozen Deep SVDD detector with statistically valid post-AD inference. The key idea is to perform inference conditional on the event that a test sample is detected as anomalous by Deep SVDD. Based on this formulation, PADI derives valid selective $p$-values that quantify the statistical significance of the detected anomalies while provably controlling the FPR at a user-specified significance level $\alpha$. Importantly, the proposed framework operates in a post hoc manner and does not require retraining or modifying the underlying anomaly detector.

\textbf{Contributions.}
The main contributions of this work are summarized as follows:



$\bullet$ We formulate anomaly assessment in Deep SVDD under the SI framework and address the fundamental double-dipping issue arising from conducting inference after AD. Based on this formulation, we derive valid selective $p$-values and theoretically establish control of the FPR at a user-specified significance level $\alpha$. The proposed PADI method operates in a post hoc manner and can be directly applied to trained and frozen Deep SVDD models without requiring any retraining or modification of the underlying detector. The proposed inference procedure additionally requires an independent set of
normal data, from which normal reference samples are drawn for computing
selective $p$-values. Furthermore, we extend PADI to deep semi-supervised AD model \citep{ruff2019deep}.

$\bullet$ We provide a GPU-accelerated implementation of PADI to improve computational efficiency. By alleviating the computational burden associated with post-AD inference, this implementation extends the practical applicability of PADI to a broader range of deep architectures beyond simple fully connected networks.

$\bullet$ We conduct extensive experiments on both synthetic and real-world benchmark datasets to validate the proposed method. The experimental results consistently demonstrate that PADI achieves reliable FPR control while maintaining superior true positive rates compared with existing methods.

\begin{figure}[t]
    \centering
    \includegraphics[width=\linewidth]{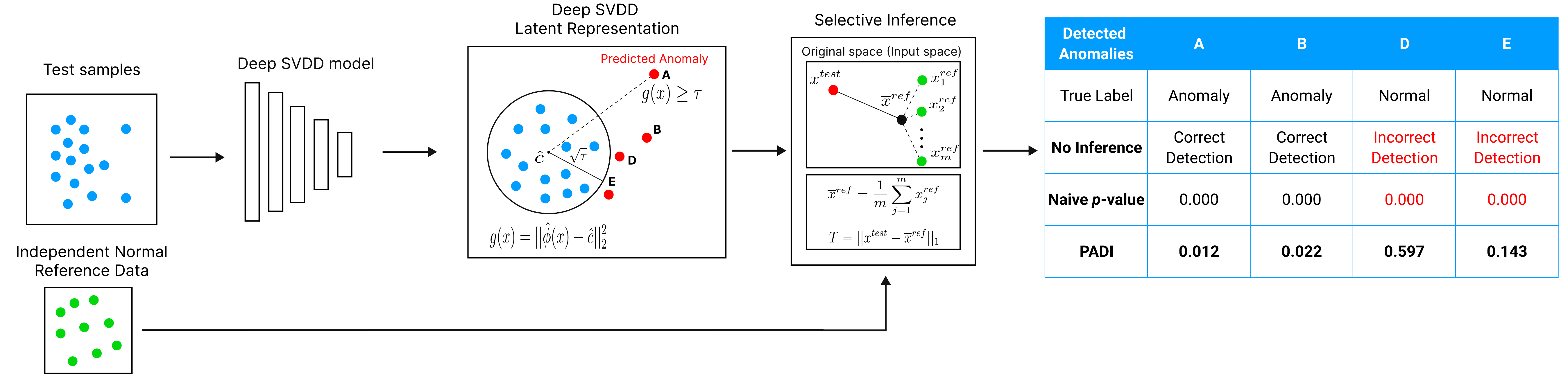}
\caption{Overview of the proposed PADI method. 
A trained and frozen Deep SVDD model first maps test instances into the latent space and selects anomalies whose distance-based anomaly scores exceed a fixed threshold. 
For each selected anomaly, PADI performs SI conditional on the anomaly-selection event induced by the Deep SVDD detector. 
Using independent normal reference samples, PADI computes a selective $p$-value that quantifies the statistical significance of the detected anomaly while accounting for the selection event. 
The right panel illustrates how PADI differs from conventional inference: unlike naive $p$-values that ignore the anomaly-selection process, PADI provides statistically valid inference for detected anomalies.}
    \label{fig:trust-svdd-example}
\vspace{-10pt}
\end{figure}


\vspace{2pt}
\textbf{Related works.} 
Unsupervised AD aims to identify abnormal or rare samples without requiring labeled anomaly instances during training. Early approaches include distance-based methods \citep{knorr2000distance}, density-based techniques such as Local Outlier Factor (LOF) \citep{breunig2000lof}, clustering-based methods \citep{jain1999data}. One-class classification methods, particularly One-Class SVM \citep{scholkopf2001estimating} and Support Vector Data Description (SVDD) \citep{tax2004support}, have also become foundational techniques due to their ability to characterize the distribution of normal data. More recently, deep learning has substantially advanced AD by enabling representation learning in high-dimensional and complex data domains. Representative deep anomaly detection approaches include autoencoder-based methods \citep{sakurada2014anomaly}, GAN-based methods such as f-AnoGAN \citep{schlegl2019f}, and Deep SVDD \citep{ruff2018deep}. Although these deep learning-based methods often demonstrate strong AD performance, most of them still lack a principled statistical framework for rigorously quantifying the significance and reliability of anomaly decisions.

Traditional statistical inference methods fail in this setting because their validity fundamentally relies on the target anomalies being predetermined prior to observing the data. When classical inference procedures are directly applied to anomalies identified by an anomaly detection (AD) algorithm, the resulting statistical tests become invalid due to selection bias, leading to the inability to properly control the FPR at the desired significance level.
Selective Inference (SI) \citep{fithian2014optimal, lee2016exact} has emerged as a promising approach for addressing the invalidity of classical post-selection inference. The core idea of SI is to conduct inference conditional on the event that a particular hypothesis has been selected, thereby removing the selection bias and restoring statistical validity in the sense that the FPR is properly controlled. Following the seminal work of \citet{lee2016exact}, SI has been extensively studied and successfully applied to a broad range of machine learning and statistical problems, including feature selection \citep{lockhart2014significance, fithian2014optimal, tibshirani2016exact, yang2016selective, suzumura2017selective, le2022more}, changepoint detection \citep{umezu2017selective, hyun2018exact, duy2020computing, jewell2022testing}, clustering \citep{lee2015evaluating, inoue2017post, gao2024selective}, image segmentation tasks \citep{tanizaki2020computing, duy2022quantifying}, saliency map analysis \citep{miwa2023valid}, and attention map interpretation in vision transformers \citep{shiraishi2024statistical}.

Statistical inference for AD has recently begun to attract attention within the SI literature. Existing studies have explored SI for anomaly testing in the context of robust regression \citep{chen2020valid, tsukurimichi2022conditional, phong2025controllable}. 
The authors of \citet{niihori2025quantifying} recently investigated the statistical significance of anomalies detected by a $k$-nearest-neighbor-based model. Their test statistic is constructed based on whether a test instance and its data-selected $k$-th nearest normal neighbor share the same underlying signal. The $k$-nearest-neighbor-based AD framework is fundamentally different from SVDD, resulting in distinct formulations of the selection event and, consequently, different challenges in developing the corresponding SI procedure.
\citet{kiet2026statistical} further investigated selective inference (SI) for autoencoder-based AD following representation-learning-based domain adaptation, in contrast to our work, which focuses on SI for AD in the conventional, non-domain-adaptation setting. Moreover, their implementation is primarily designed for relatively simple operations in traditional fully connected neural networks, such as ReLU, and their experimental evaluation is conducted on tabular datasets. Consequently, their method cannot be directly applied to the more complex operations commonly encountered in CNN architectures, such as Conv2D, BatchNorm, and MaxPool, or to image data, which are also considered in our work.
To the best of our knowledge, no existing study has explored the SI framework for quantifying the statistical significance of AD results produced by Deep SVDD.

\section{Problem Statement}
\label{sec:problem_statement}

In this section, we formalize the post-AD inference for a trained Deep SVDD method.
Each input instance is represented as a vector in \(\mathbb{R}^D\), where \(D\) denotes the
input dimension. For tabular data, \(D\) is the number of numerical features. For
image data, an input instance may be represented as a vectorized image; for
example, a grayscale image patch of size \(H\times W\) corresponds to \(D=HW\),
whereas an image with \(C\) channels corresponds to \(D=HWC\).
The Deep SVDD model is trained independently of the post-AD inference task and remains fixed during the test-time analysis. Given a test instance, the trained detector identifies it as anomalous according to its distance from a center in the latent representation space. Our objective is neither to modify the detector nor to perform inference on the training procedure itself. Instead, once a test sample has been identified as anomalous, we aim to statistically validate the detection result by quantifying the statistical significance of the discrepancy between the selected input instance and an independent reference set of observed normal instances.

\subsection{A trained Deep SVDD model and its anomaly detection event}
\label{subsec:frozen_detector_problem}
Let $
\hat{\phi}:\mathbb{R}^D\to\mathbb{R}^p
$
denote the trained encoder, where \(p\) is the latent dimension. Let
$
\hat{\vc}\in\mathbb{R}^p
$
denote the fixed latent center. For any input vector \(\vx^{\rm test} \in\mathbb{R}^D\), the
trained Deep SVDD assigns the following anomaly score:
\begin{equation}
g(\vx^{\rm test})
=
\left\|
\hat{\phi}(\vx^{\rm test})-\hat{\vc}
\right\|_2^2.
\label{eq:score_general}
\end{equation}
Given a fixed threshold \(\tau>0\), the detector classifies \(\vx^{\rm test}\) as anomalous
when
$
g(\vx^{\rm test})\ge \tau.
\label{eq:selection_rule}
$
Equivalently, we define the AD selection event through the following selection mapping:
\begin{equation}
\mathcal{A}:\mathbb{R}^D\to\{0,1\},
\qquad
\mathcal{A}(\vx^{\rm test})
=
\mathbb{I}\{g(\vx^{\rm test})\ge \tau\},
\label{eq:selection_mapping_problem}
\end{equation}
where $\mathbb{I}\{\cdot\}$ denotes the indicator function. The mapping
$\mathcal{A}(\vx^{\rm test})=1$ indicates that the input instance $\vx^{\rm test}$ is selected by
the trained Deep SVDD detector as an anomalous sample, whereas
$\mathcal{A}(\vx^{\rm test})=0$ corresponds to a non-anomalous decision.

The original Deep SVDD formulation includes both hard-boundary and
soft-boundary variants, which differ only in the training objective used to
learn the encoder. Since PADI focuses on the post-selection
inference stage after the detector has been trained, the proposed framework is
independent of the specific training variant. Once the encoder $\hat{\phi}$,
center $\hat{c}$, and anomaly-selection threshold $\tau$ are fixed, PADI applies the same
selective inference procedure to either a hard-boundary or soft-boundary Deep
SVDD detector. Here, $\tau$ denotes the anomaly-selection threshold used in our inference
procedure and should not be confused with the radius parameter $R$ in the
soft-boundary Deep SVDD formulation. In our framework, $\tau$ is a fixed
threshold that determines the anomaly selection event.

%
%


\subsection{Statistical hypothesis testing for the detected anomaly}

\label{subsec:hypothesis_testing_problem}

To formulate the statistical inference problem, we regard the test instance
$\vx^{\rm test}$ as an observed realization of the random vector
\begin{align*}
\vX^{\mathrm{test}}
=
\vs^{\mathrm{test}}
+
\veps^{\mathrm{test}},
\end{align*}
where $\vs^{\mathrm{test}} \in \mathbb{R}^D$ denotes the unknown underlying
signal vector, and $\veps^{\mathrm{test}}$ represents an additive noise vector
following the Gaussian distribution
$\mathcal{N}(\mathbf{0},\Sigma)$. Here,
$\Sigma \in \mathbb{R}^{D \times D}$ denotes the covariance matrix, which is
assumed to be known a priori or estimated from an independent dataset. In practice, this assumption is reasonable in anomaly detection scenarios because normal samples are typically much more abundant than anomalous samples. 
Therefore, an independent set of normal observations can be used to estimate $\Sigma$ without requiring additional anomalous data.

Additionally, we consider a collection of reference instances known to be
normal:
$\vX^{1, \mathrm{ref}}, \dots, \vX^{m, \mathrm{ref}}$
where each reference instance is modeled as
\begin{align*}
\vX^{j, \mathrm{ref}}
=
\bm s^{\mathrm{ref}}
+
\veps^{j, \mathrm{ref}},
\qquad
j \in [m]= \{1,\dots,m\},
\end{align*}
with $\bm s^{\mathrm{ref}} \in \mathbb{R}^D$ denoting the underlying normal
signal vector. The noise vector $\veps^{j, \mathrm{ref}}$ is assumed to follow
the Gaussian distribution
$
\veps^{j, \mathrm{ref}}
\sim
\mathcal{N}(\mathbf{0},\mSigma)
$,
independently across $j$. The normal reference instances used for selective inference are assumed to be independent of the dataset used for estimating $\Sigma$. 
This separation ensures that the covariance matrix is treated as a fixed quantity during the selective inference procedure, which is consistent with the theoretical derivation.

Our objective is to determine whether the underlying signal associated with
the test instance differs statistically significantly from that of the
reference normal instances. This objective can be formulated as a hypothesis testing problem, consisting of the following null
hypothesis $H_0$ and alternative hypothesis $H_1$:
\begin{align*}
{\rm H}_0:\ \vs^{\mathrm{test}}=\vs^{\rm ref}
\qquad\text{vs.}\qquad
{\rm H}_1:\ \vs^{\mathrm{test}}\neq \vs^{\rm ref}.
\label{eq:null_signal_mean}
\end{align*}
The test statistic for evaluating the above hypotheses is defined as follows:
\begin{equation}
T\!\left(
\vX^{\mathrm{test}},
\vX^{1, \mathrm{ref}},
\dots,
\vX^{m, \mathrm{ref}}
\right)
=
\left\|
\vX^{\mathrm{test}}
-
{\bar{\vX}}^{\mathrm{ref}}
\right\|_1,
\label{eq:test_stat_general}
\end{equation}
where ${\bar{\vX}}^{\mathrm{ref}}
=
\frac{1}{m}\sum_{j=1}^m \vX^{j, \mathrm{ref}}$.
We would like to note that the choice of the $\ell_1$-norm is motivated by the desire to align our procedure with the seminal SI framework of \cite{lee2016exact}. Although the $\ell_2$-norm could also be considered, its use would result in a test statistic of quadratic form. Such a statistic falls outside the theoretical framework established in \cite{lee2016exact}, upon which our SI procedure is based, where the test statistic is required to be a linear contrast of the data.

\subsection{Decision making based on 
$p$-values and challenges}
\label{subsec:challenges}

After obtaining the test statistic in \eq{eq:test_stat_general}, the next step
is to compute the corresponding $p$-value. Given a significance level
$\alpha \in [0,1]$ (e.g., $\alpha = 0.05$), we reject the null hypothesis and
conclude that the test instance is anomalous if the computed $p$-value is less
than or equal to $\alpha$. Conversely, if the $p$-value exceeds $\alpha$, we
conclude that there is insufficient statistical evidence to determine that the
test instance is anomalous.

The $p$-value is defined as follows:
\begin{equation}
p
=
\mathbb{P}_{{\rm H}_0}
\Big(
\left | T\!\left(
\vX^{\mathrm{test}},
\vX^{1, \mathrm{ref}},
\dots,
\vX^{m, \mathrm{ref}}
\right)
\right |
\ge
\left |T\!\left(
\bm x^{\mathrm{test}},
\bm x^{1, \mathrm{ref}},
\dots,
\bm x^{m, \mathrm{ref}}
\right)
\right |
\Big),
\label{eq:pvalue}
\end{equation}
where $\bm x^{j, \rm ref}$ denotes an observed realization of the random vector
of $\vX^{j, \rm ref}$ for each $j \in [m]$, respectively.
Unfortunately, computing the $p$-value in \eq{eq:pvalue} is intractable because the test statistic depends on both the observed data and the AD result produced by the Deep SVDD model, for which no direct computation is available.
A conventional (naive) approach computes the $p$-value while ignoring the fact that the test instance has been selected as anomalous by the trained Deep SVDD model. As a consequence, the resulting naive $p$-value is statistically invalid and fails to properly control the FPR. In particular, it does \emph{not} satisfy the fundamental validity criterion required of a valid $p$-value:
\begin{align} \label{eq:valid_p_value}
	\mathbb{P} \Big (
	\underbrace{p{-\rm value} \leq \alpha \mid {\rm H}_{0} \text{ is true }}_{\text{a false positive}}
	\Big) = \alpha, ~~ \forall \alpha \in [0, 1],
\end{align} 
In the next section, we introduce a \emph{selective $p$-value} for statistically
testing anomalies detected by Deep SVDD, which satisfies the aforementioned
validity criterion.

\section{Proposed PADI Method}
\label{sec:proposed_method}

In this section, we introduce PADI, an SI-based method for computing statistically valid $p$-values for anomalies detected by a trained Deep SVDD model. Rather than relying on the unconditional null distribution of the test statistic in \eqref{eq:test_stat_general}, the proposed method characterizes the null distribution conditional on the data-dependent AD event induced by the Deep SVDD detector.

\subsection{Representation of the test statistic as a linear contrast of the data vector}
\label{subsec:stacked_sign_problem}

%

Let $\bm Y$ denote the \((m+1)D\)-dimensional stacked vector formed by concatenating the test instance and the reference instances:
\begin{equation}
\bm Y
=
\operatorname{vec}
\left(
\vX^{\mathrm{test}},
\vX^{1, \mathrm{ref}},
\dots,
\vX^{m, \mathrm{ref}}
\right)
\in\mathbb{R}^{(m+1)D},
\label{eq:stacked_general}
\end{equation}
where \(\operatorname{vec} (\cdot)\) denotes the operation that concatenates multiple vectors into a single column vector.
We aim to represent the test statistic as a
linear contrast of the vector \(\bm Y\). To this end, define the
coordinate-wise sign pattern
\begin{equation}
\cS(\bm Y)
=
\operatorname{sign}
\left(
\vX^{\mathrm{test}}
-
{\bar{\vX}}^{\mathrm{ref}}
\right)
\in\{-1,1\}^D
\label{eq:sign_mapping}
\end{equation}
Then, the test statistic in \eq{eq:test_stat_general} admits
the linear representation
\begin{equation}
T(\bm Y)
=
\veta^\top \bm Y,
\label{eq:linear_representation}
\end{equation}
where $\bm \eta$ is the direction vector of the test statistic, defined as:
\begin{equation}
\veta
=
\begin{pmatrix}
\cS(\bm Y)\\
-\frac{1}{m}\cS(\bm Y)\\
\vdots\\
-\frac{1}{m}\cS(\bm Y)
\end{pmatrix}
\in\mathbb{R}^{(m+1)D}.
\label{eq:eta}
\end{equation}

\subsection{Conditional distribution of the test statistic and the proposed selective $p$-value}
\label{subsec:selective_pvalue_method}

To compute a statistically valid $p$-value, we need to characterize the sampling distribution of the test statistic in \eqref{eq:test_stat_general}. To this end, we leverage the framework of conditional SI \citep{lee2016exact}. Specifically, we consider the conditional distribution of the test statistic given the data-dependent selection event:
\begin{align} \label{eq:conditional_distribution}
	\bP
	\left ( 
		\bm \eta^\top \bm Y 
		\mid
		\cA(\vX^{\rm test})=\cA(\bm x^{\rm test}),\,
\cS(\bm Y)=\cS(\bm y)
	\right ). 
\end{align}
Here, the first condition
\(
\cA(\vX^{\rm test})=\cA(\bm x^{\rm test})
\)
represents the event that the AD result for the random vector
\(\vX^{\rm test}\) coincides with the AD result obtained from the observed
data \(\bm x^{\rm test}\).
The second condition
\(
\cS(\bm Y)=\cS(\bm y)
\)
represents the event that the sign pattern defined in
\eqref{eq:sign_mapping} for the random vector \(\bm Y\) coincides with the
observed sign pattern for \(\bm y\).

Based on the distribution in \eq{eq:conditional_distribution}, we introduce the selective $p$-value defined as:
\begin{equation}
p^{\mathrm{selective}}
=
\mathbb P_{H_0}
\Big(
|\veta^\top \bm Y|
\ge
|\veta^\top \bm y|
\mid
\cA(\vX^{\rm test})=\cA(\bm x^{\rm test}),
\,
\cS(\bm Y)=\cS(\bm y), 
\, 
\cQ(\bm Y) = \cQ(\bm y)
\Big),
\label{eq:selective_pvalue}
\end{equation}
where $\cQ(\bm Y)$ denotes sufficient statistic of the nuisance parameter, defined as:
\begin{equation}
 \cQ(\bm Y) 
=
\left(\mI_{(m+1)D}-\vb\veta^\top\right)\bm Y,
\quad
\vb
=
\frac{\tilde{\mSigma} \veta}{\veta^\top \tilde{\mSigma}\veta},
\quad 
\tilde{\mSigma}= I_{m+1}\otimes \mSigma.
\label{eq:Q_and_b_main}
\end{equation}

\begin{remark}
The quantity \(\mathcal{Q}(\bm Y)\) acts as a sufficient statistic for the
nuisance parameter, i.e., a parameter that influences the null distribution but
is not of direct inferential interest. To properly characterize the null
distribution, the effect of this nuisance parameter must be eliminated. In our
framework, this is accomplished by conditioning on the sufficient statistic
\(\mathcal{Q}(\bm Y)\). This conditioning step is primarily technical and is
standard in the SI literature (see Sec.~5 and Eq.~(5.2) of
\mbox{\cite{lee2016exact}}; \mbox{\cite{fithian2014optimal}}).
\end{remark}

Intuitively, Eq.~\eqref{eq:Q_and_b_main} decomposes the data vector $\bm Y$ into the scalar contrast of inferential interest $z = \veta^\top \bm Y$ and the remaining nuisance component $\mathcal{Q}(\bm Y)$, since $\bm Y = \mathcal{Q}(\bm Y) + \vb z$. Conditioning on $\mathcal{Q}(\bm Y) = \mathcal{Q}(\bm y)$ fixes this nuisance variation while leaving only $z$ free to vary. Consequently, the remaining randomness is restricted to the one-dimensional affine line $\bm Y = \va + \vb z$, as formalized in Theorem~\ref{thm:data_line}.

\begin{theorem}
\label{thm:validity_main}
The selective p-value proposed in \eq{eq:selective_pvalue} satisfies the property of a valid $p$-value:
\[
\mathbb P_{H_0}
\left(
p^{\mathrm{selective}}\le \alpha
\right)
=
\alpha, \quad \forall \alpha \in [0, 1]
\]
\end{theorem}

\begin{proof}
The proof is given in Appendix~\ref{app:proof_validity}.
\end{proof}

\subsection{Tractable characterization of the conditioning event for selective $p$-value computation}

The selective $p$-value in \eqref{eq:selective_pvalue} requires evaluating the conditional distribution
of the test statistic under the event that the anomaly selection and the
associated conditioning information, including the sign pattern induced by
the $\ell_1$ test statistic, remain unchanged. Therefore, computing the
selective p-value requires an explicit characterization of the set of data
vectors satisfying these conditioning constraints. Since directly
characterizing this event in the original high-dimensional space is
intractable, we show that it can be reduced to a one-dimensional truncation
problem along the affine line induced by the selective inference formulation.

Let us define the set of vectors $\bm Y$ satisfying the conditions in \eqref{eq:selective_pvalue} as
\begin{equation}
\cD
=
\Bigl\{
\bm Y \in \mathbb{R}^{(m+1)D} \mid 
\cA(\vX^{\rm test})=\cA(\bm x^{\rm test}),
\,
\cS(\bm Y)=\cS(\bm y), 
\, 
\cQ(\bm Y) = \cQ(\bm y)
\Bigr\}.
\label{eq:condition_event_set}
\end{equation}

\begin{theorem} \label{thm:data_line}
The set \(\mathcal{D}\) in \eqref{eq:condition_event_set} can be expressed as
\begin{align} \label{eq:condition_event_set_line}
    \mathcal{D}
    =
    \left\{
    \bm Y (z)
    =
    \bm a + \bm b z
    \;\middle|\;
    z \in \mathcal{Z}
    \right\},
\end{align}
where \(\bm a =  \cQ(\bm y)\), \(\bm b\) is defined in
\eqref{eq:Q_and_b_main}, and
\begin{align} \label{eq:cZ}
    \mathcal{Z}
    =
    \left\{
    z \in \mathbb{R}
    \;\middle|\;
    \cA(\vX^{\rm test}(z))=\cA(\bm x^{\rm test}),
\,
\cS(\bm Y (z))=\cS(\bm y) 
    \right\}.
\end{align}
Here, we note that \(\vX^{\rm test}(z)\) corresponds to the first \(D\)
components of the vector \(\bm Y(z)\).
\end{theorem}

\begin{proof}
The proof is provided in Appendix~\ref{app:proof_data_line}.
\end{proof}

Theorem~\ref{thm:data_line} shows that the SI problem can be
reduced from the original high-dimensional space to the scalar parameter space
\(\mathcal{Z}\). Consequently, rather than analyzing the entire
\((m + 1)D\)-dimensional space, it is sufficient to characterize
the truncation region \(\mathcal{Z}\) in a one dimensional space. Once \(\mathcal{Z}\) is obtained, the
selective \(p\)-value can be computed directly.

\subsection{Identification of the truncation region \(\mathcal{Z}\)}
\label{subsec:characterization_Z}
\label{subsec:characterization_main}

To compute the selective $p$-value in \eqref{eq:selective_pvalue}, we must
identify the truncation region $\mathcal{Z}$ defined in \eqref{eq:cZ}.
However, this region cannot be determined directly due to the complexity of the Deep SVDD selection event.
To address this, we exploit the piecewise-linear structure of the encoder to
 identify $\mathcal{Z}$ through as follows:
\begin{itemize}
\item We decompose $\mathcal{Z}$ into two sub-problems: a sign-pattern
  constraint $\mathcal{Z}_{\mathrm{sign}}$ and a Deep SVDD
  anomaly-selection constraint $\mathcal{Z}_{\mathrm{AD}}$.
\item We show that $\mathcal{Z}_{\mathrm{sign}}$ reduces to $D$ linear
  inequalities in the scalar~$z$, and that $\mathcal{Z}_{\mathrm{AD}}$
  reduces to a quadratic inequality on each affine region of the frozen
  encoder.
\item We construct $\mathcal{Z}$ by intersecting
  $\mathcal{Z}_{\mathrm{sign}}$ with the union of the local solutions
  across all affine regions intersected by the one-dimensional path
  $\vX^{\mathrm{test}}(z)$.
\end{itemize}

\begin{assumption}\label{assump:pwa}
Following \citet{niihori2025quantifying}, we assume that the frozen encoder
$\hat{\phi}:\mathbb{R}^D\to\mathbb{R}^p$ is a  piecewise-affine function. That is, the input space $\mathbb{R}^D$ can be partitioned
into finitely many polyhedral regions, and on each region $\mathcal{P}$ the
encoder acts as an affine map
$\hat{\phi}(\vx) = \mL_{\mathcal P}\vx+\vbeta_{\mathcal P}$ for
$\vx\in\mathcal P$, where
$\mL_{\mathcal P}\in\mathbb{R}^{p\times D}$ and
$\vbeta_{\mathcal P}\in\mathbb{R}^p$ are fixed for each region
$\mathcal{P}$.
\end{assumption}

\begin{remark}
Assumption~\ref{assump:pwa} is satisfied by neural networks composed of affine
layers (fully connected, convolution, batch normalization in inference mode) and
piecewise-linear activations (ReLU, LeakyReLU) together with max-pooling.
Since the Deep SVDD encoder used in this work consists exclusively of such
layers, this assumption holds by construction.
\end{remark}

\paragraph{Decomposition of $\mathcal{Z}$.}
By Theorem~\ref{thm:data_line}, after conditioning on the nuisance sufficient
statistic, the data vector $\bm Y$ is restricted to the one-dimensional affine
line $\bm Y(z) = \bm a + \bm b z$, and the test sample varies as
$\vX^{\mathrm{test}}(z)=\bm a^{\mathrm{test}}+\bm b^{\mathrm{test}}z$.
We decompose the truncation region as
\begin{equation}\label{eq:Z_decompose}
\mathcal{Z}
=
\mathcal{Z}_{\mathrm{sign}}
\cap
\mathcal{Z}_{\mathrm{AD}},
\end{equation}
where $\mathcal{Z}_{\mathrm{sign}}$ enforces the sign pattern used to
linearize the $\ell_1$-discrepancy, and $\mathcal{Z}_{\mathrm{AD}}$ enforces
the Deep SVDD anomaly-selection event
$\cA(\vX^{\mathrm{test}}(z))=\cA(\bm x^{\rm test})$.
We characterize $\mathcal{Z}_{\mathrm{sign}}$ and $\mathcal{Z}_{\mathrm{AD}}$
through the following two lemmas.

\begin{lemma}[Sign-pattern constraint] \label{lem:z_sign}
The sign-feasible set $\mathcal{Z}_{\mathrm{sign}}$ is characterized by a set of $D$ linear inequalities with respect to $z$, and can be obtained as a single interval.
\end{lemma}

\begin{proof}[Proof]
The sign pattern $\cS(\bm Y(z))=\cS(\bm y)$ requires, for each coordinate
$u=1,\dots,D$:
\[
\cS_u(\bm y)
\left(
\alpha_u + \gamma_u z
\right) > 0,
\]
where $\alpha_u$ and $\gamma_u$ are the intercept and slope of
$X_u^{\mathrm{test}}(z)-\bar{X}_u^{\mathrm{ref}}(z)$, respectively.
Each coordinate produces one linear inequality in $z$, yielding a single intersection interval. The detailed derivation is provided in Appendix~\ref{appsubsec:Z_sign}.
\end{proof}

\begin{lemma}[Deep SVDD selection constraint] \label{lem:z_ad}
Let $\mathfrak{P}_{\mathrm{line}}$ denote the finite collection of affine
regions of $\hat{\phi}$ that are intersected by the path
$\vX^{\mathrm{test}}(z)$ as $z$ varies over $\mathbb{R}$.
Under Assumption~\ref{assump:pwa}, on each affine region $\mathcal{P}\in\mathfrak{P}_{\mathrm{line}}$, the Deep SVDD anomaly-selection event reduces to a scalar quadratic inequality in $z$.
\end{lemma}

\begin{proof}[Proof]
By Assumption~\ref{assump:pwa}, for $\vX^{\mathrm{test}}(z)\in\mathcal{P}$, the encoder output is affine in $z$: $\hat{\phi}(\vX^{\mathrm{test}}(z)) = \mL_{\mathcal{P}}(\va^{\mathrm{test}} + \vb^{\mathrm{test}}z) + \vbeta_{\mathcal{P}}$. Substituting this into the anomaly score definition $g(\vX^{\mathrm{test}}(z)) = \|\hat{\phi}(\vX^{\mathrm{test}}(z)) - \hat{\vc}\|_2^2$ yields the quadratic function of $z$:
\begin{equation}\label{eq:score_quadratic}
g_{\mathcal{P}}(z)
=
\|\vu_0(\mathcal{P})+\vu_1(\mathcal{P})z\|_2^2
=
\kappa_2(\mathcal{P}) z^2+\kappa_1(\mathcal{P}) z+\kappa_0(\mathcal{P}),
\end{equation}
where $\vu_0(\mathcal{P}) = \mL_{\mathcal{P}}\va^{\mathrm{test}} + \vbeta_{\mathcal{P}} - \hat{\vc}$ and $\vu_1(\mathcal{P}) = \mL_{\mathcal{P}}\vb^{\mathrm{test}}$ are constant vectors determined by the affine map on region $\mathcal{P}$, and the coefficients are $\kappa_2(\mathcal{P})=\|\vu_1(\mathcal{P})\|_2^2$, $\kappa_1(\mathcal{P})=2\vu_0(\mathcal{P})^\top\vu_1(\mathcal{P})$, $\kappa_0(\mathcal{P})=\|\vu_0(\mathcal{P})\|_2^2$. The anomaly-selection event $g(\vX^{\mathrm{test}}(z))\ge\tau$ restricted to region $\mathcal{P}$ therefore reduces to the quadratic inequality $g_{\mathcal{P}}(z) \ge \tau$. The explicit expressions and derivations are provided in Appendix~\ref{appsubsec:Z_region} and \ref{appsubsec:Z_score}.
\end{proof}

\begin{remark}
A similar piecewise-affine characterization for autoencoder-based models in the context of Selective Inference has been recently studied by \citet{kiet2026statistical}. 
Although formulated for Deep SVDD, the quadratic-constraint argument in Lemma~\ref{lem:z_ad} does not depend on the Deep SVDD training objective and applies more generally to any fixed encoder satisfying Assumption~\ref{assump:pwa} whose anomaly decision is obtained by comparing the \(\ell_2\) distance between its representation and a fixed center with a fixed nonnegative threshold.
\end{remark}

\paragraph{Combined truncation region.}
Combining the results from Lemma~\ref{lem:z_sign} and Lemma~\ref{lem:z_ad}
via \eqref{eq:Z_decompose}, we now construct the full truncation region.
Since the test sample $\vX^{\mathrm{test}}(z)$ may pass through different
affine regions as $z$ varies, the anomaly-selection constraint
$\mathcal{Z}_{\mathrm{AD}}$ is obtained by considering each region
$\mathcal{P}\in\mathfrak{P}_{\mathrm{line}}$ separately. For each such
region, define:

\begin{itemize}
\item $\mathcal{Z}_{\mathrm{region}}(\mathcal{P})$: the set of $z$-values
  for which $\vX^{\mathrm{test}}(z)$ lies in $\mathcal{P}$ (determined by
  linear inequalities; see Appendix~\ref{appsubsec:Z_region});
\item $\mathcal{Z}_{\mathrm{score}}(\mathcal{P})$: the set of $z$-values
  satisfying the quadratic anomaly-score constraint
  $g_{\mathcal{P}}(z)\ge\tau$ from \eqref{eq:score_quadratic}.
\end{itemize}
The full truncation region is then
\begin{equation}\label{eq:Z_combined}
\mathcal{Z}
=
\mathcal{Z}_{\mathrm{sign}}
\cap
\bigcup_{\mathcal{P}\in\mathfrak{P}_{\mathrm{line}}}
\left(
\mathcal{Z}_{\mathrm{region}}(\mathcal{P})
\cap
\mathcal{Z}_{\mathrm{score}}(\mathcal{P})
\right).
\end{equation}

A geometric illustration of the truncation region $\mathcal{Z}$ on the
one-dimensional line parameterized by $z$ is shown in
Fig.~\ref{fig:truncation_region}.
%
%
Because $\mathfrak{P}_{\mathrm{line}}$ is finite and the sets in Eq.~\eqref{eq:Z_combined} are defined by finitely many scalar linear or quadratic inequalities, $\mathcal{Z}$ is a finite union of intervals. After merging all overlapping interval pieces, we write
\begin{equation}\label{eq:Z_intervals}
\mathcal{Z}
=
\bigcup_{\ell=1}^{K}
\left[\underline{z}_{\ell},\overline{z}_{\ell}\right],
\qquad K<\infty.
\end{equation}
The resulting intervals are pairwise disjoint, and $K$ is the number of intervals remaining after merging. Thus, $K$ is obtained directly during the construction of $\mathcal{Z}$.

\textbf{Relation to over-conditioning.} Let $\mathcal{P}_{\mathrm{obs}}\in \mathfrak{P}_{\mathrm{line}}$ denote the affine region containing the observed test sample $\vX^{\mathrm{test}}$. The over-conditioning (OC) baseline additionally conditions on this observed affine region, leading to the truncation region
\[
\mathcal{Z}_{\mathrm{OC}}
=
\mathcal{Z}_{\mathrm{sign}}
\cap
\mathcal{Z}_{\mathrm{region}}(\mathcal{P}_{\mathrm{obs}})
\cap
\mathcal{Z}_{\mathrm{score}}(\mathcal{P}_{\mathrm{obs}}).
\]
In contrast, PADI does not condition on the affine-region identity and instead aggregates all affine regions intersected by $\vX^{\mathrm{test}}(z)$, as in \eqref{eq:Z_combined}. Hence, $\mathcal{Z}_{\mathrm{OC}}\subseteq \mathcal{Z}$. OC is computationally simpler because it considers only the observed affine region, whereas PADI requires a line search across multiple affine regions. This reflects a computational--statistical trade-off: OC introduces additional conditioning that may reduce statistical power, while PADI avoids this conditioning at the cost of additional computation. This distinction is analogous to the single-polyhedron versus union-of-polyhedra cases discussed by \citet{lee2016exact}.

Finally, with the full truncation region
\(\mathcal{Z}\) represented as the finite union of intervals in
Eq.~\eqref{eq:Z_intervals}, the selective \(p\)-value in
Eq.~\eqref{eq:selective_pvalue} can be evaluated as the two-sided tail
probability of \(z=\veta^\top\bm Y\) under its null Gaussian distribution
truncated to \(\mathcal{Z}\).

\begin{figure*}[t]
    \centering
    \includegraphics[width=.9\textwidth]{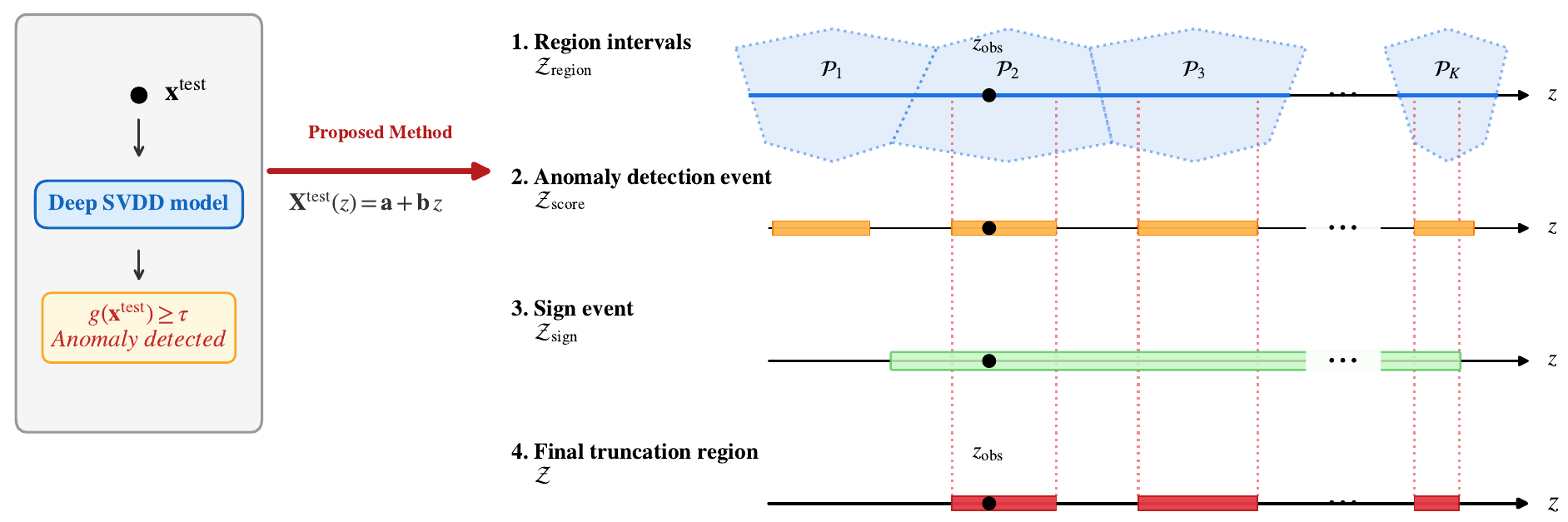}
    \caption{
        \textbf{Geometric illustration of the truncation region $\mathcal{Z}$ on the 1D line parametrized by $z$.} 
        The final region is constructed through four steps: 
        (1) $\mathcal{Z}_{\mathrm{region}}$ (blue) identifies the intervals where the line intersects the affine regions $\mathcal{P}_1, \dots, \mathcal{P}_K$. 
        (2) $\mathcal{Z}_{\mathrm{score}}$ (orange) defines the intervals where the quadratic anomaly score exceeds the threshold $\tau$, computed independently per region. 
        (3) $\mathcal{Z}_{\mathrm{sign}}$ (green) enforces the global sign-pattern constraint required to linearize the $\ell_1$-norm discrepancy. 
        (4) The final truncation region $\mathcal{Z}$ (red) is the intersection of the sign constraint with the union of region-specific valid scores. The observed test sample $z_{\mathrm{obs}}$ lies within this final valid set.
    }
    \label{fig:truncation_region}
\end{figure*}

\subsection{Algorithmic Summary of PADI}

Sections 3.1--3.4 establish the statistical formulation and characterize the truncation region required for PADI. For clarity and reproducibility, Algorithm 1 summarizes the complete test-time procedure from a Deep SVDD anomaly decision to the selective $p$-value.

\begin{algorithm}[t]
{\small
\caption{Post-Anomaly Detection Inference (PADI) for Deep SVDD}
\label{alg:padi}

\textbf{Input:}
test instance $\vx^{\mathrm{test}}$;
independent normal reference instances
$\{\vx^{j,\mathrm{ref}}\}_{j=1}^m$;
frozen Deep SVDD encoder $\hat{\phi}$, fixed center $\hat{\vc}$,
fixed threshold $\tau$;
covariance matrix $\mSigma$; significance level $\alpha$. \\

\textbf{Output:}
selective $p$-value $p^{\mathrm{selective}}$ and significance decision
if $\vx^{\mathrm{test}}$ is selected as anomalous; otherwise,
``not selected as anomalous.''

\begin{algorithmic}[1]
\State Compute
$g_{\mathrm{obs}}
\gets
g(\vx^{\mathrm{test}})
=
\lVert
\hat{\phi}(\vx^{\mathrm{test}})-\hat{\vc}
\rVert_2^2$.

\If{$g_{\mathrm{obs}} < \tau$}
    \State \Return ``not selected as anomalous''
\EndIf

\vspace{4pt}
\Statex \hspace{-12pt}
\emph{// Phase I: Construct the one-dimensional SI line}

\State Set
${\bar{\vx}}^{\mathrm{ref}}
\gets
\frac{1}{m}\sum_{j=1}^{m}\vx^{j,\mathrm{ref}}$.

\State Form
$\vy
\gets
\operatorname{vec}
(\vx^{\mathrm{test}},
 \vx^{1,\mathrm{ref}},
 \dots,
 \vx^{m,\mathrm{ref}})$
and compute
$\cS(\vy)
=
\operatorname{sign}
(\vx^{\mathrm{test}}-{\bar{\vx}}^{\mathrm{ref}})$.

\State Construct $\veta$ from $\cS(\vy)$ as in \eqref{eq:eta}.

\State Set
$\tilde{\mSigma}
\gets
I_{m+1}\otimes\mSigma$,
$\vb
\gets
\tilde{\mSigma}\veta/
(\veta^\top\tilde{\mSigma}\veta)$,
$z_{\mathrm{obs}}
\gets
\veta^\top\vy$,
and
$\va
\gets
\cQ(\vy)
=
\vy-\vb z_{\mathrm{obs}}$.

\State Define
$\bm Y(z)\gets\va+\vb z$
and extract $\vX^{\mathrm{test}}(z)$.

\vspace{4pt}
\Statex \hspace{-12pt}
\emph{// Phase II: Identify the truncation region}

\State Compute
$\mathcal Z_{\mathrm{sign}}$
from the $D$ sign inequalities in Lemma~\ref{lem:z_sign}.

\State Initialize
$\mathcal Z_{\mathrm{AD}}\gets\emptyset$.

\For{each affine region
$\mathcal P\in\mathfrak P_{\mathrm{line}}$
encountered by the line-search procedure}

    \State Determine
    $\mathcal Z_{\mathrm{region}}(\mathcal P)$.

    \State Form the local quadratic score
    $g_{\mathcal P}(z)$
    as in \eqref{eq:score_quadratic}.

    \State Compute
    $\mathcal Z_{\mathrm{score}}(\mathcal P)
    \gets
    \{z\in\mathbb{R}:g_{\mathcal P}(z)\ge\tau\}$.

    \State
    $\mathcal Z_{\mathrm{AD}}
    \gets
    \mathcal Z_{\mathrm{AD}}
    \cup
    \bigl(
    \mathcal Z_{\mathrm{region}}(\mathcal P)
    \cap
    \mathcal Z_{\mathrm{score}}(\mathcal P)
    \bigr)$.

\EndFor

\State Set
$\mathcal Z
\gets
\mathcal Z_{\mathrm{sign}}
\cap
\mathcal Z_{\mathrm{AD}}$
and represent $\mathcal Z$ as a finite union of disjoint intervals.

\vspace{4pt}
\Statex \hspace{-12pt}
\emph{// Phase III: Compute p-value}

%
\State 
Compute
$
p^{\mathrm{selective}}
$

\State \Return
$p^{\mathrm{selective}}$
and reject $H_0$ if
$p^{\mathrm{selective}}\le\alpha$.

\end{algorithmic}
}
\end{algorithm}

\section{Extension to Deep Semi-Supervised Anomaly Detection}
\label{sec:extension_deep_sad}

The proposed method also applies to Deep Semi-Supervised
Anomaly Detection (Deep SAD) \citep{ruff2019deep}. Deep SAD differs from Deep
SVDD only during training: it incorporates a small amount of labeled data,
encouraging known anomalies to be mapped far from the latent center and known
normal samples to be mapped close to it. After training, the frozen Deep SAD
detector uses exactly the same scoring rule as Deep SVDD: a test sample is
assigned the anomaly score
$g(\vx^{\mathrm{test}})=\|\hat{\phi}(\vx^{\mathrm{test}})-\hat{\vc}\|_2^2$
and is declared anomalous when $g(\vx^{\mathrm{test}})\ge\tau$.
Because the functional form of the anomaly score and the selection rule are
identical to those of Deep SVDD, the post-selection inference problem has the
same structure once the encoder, center, and threshold are frozen.
Specifically, the inferential target, the sign-pattern conditioning, the
nuisance conditioning, and the one-dimensional affine-line reduction
(Theorem~\ref{thm:data_line}) all remain unchanged. The only difference is
that the frozen encoder $\hat{\phi}$ is obtained from Deep SAD training rather
than Deep SVDD training; the anomaly-selection event and its mathematical
characterization are of the same form.

Under Assumption~\ref{assump:pwa}, the frozen Deep SAD encoder is
piecewise affine, so the anomaly score again becomes a quadratic function of
$z$ within each affine region along the nuisance-conditioned line.
Consequently, the Deep SAD truncation region, denoted by
$\mathcal{Z}_{\mathrm{SAD}}$, is a finite union of intervals. Thus, the Deep SAD case can be handled by the same procedure summarized in Algorithm~1, with the frozen Deep SVDD encoder replaced by the frozen Deep SAD encoder and $\mathcal{Z}$ replaced by $\mathcal{Z}_{\mathrm{SAD}}$.
\begin{corollary}
Under the Gaussian test-reference model and the frozen piecewise-affine encoder
assumption, the Deep SAD selective p-value satisfies
\[
    \mathbb P_{\rm H_0}
    \left(
        p^{\mathrm{selective}}_{\mathrm{SAD}} \le \alpha
    \right)
    =
    \alpha,
    \qquad
    \forall \alpha\in[0,1].
\]
\end{corollary}

Thus, the Deep SAD case can be handled by the same PADI procedure, with the
Deep SVDD truncation region replaced by the corresponding Deep SAD truncation
region \(\mathcal Z_{\mathrm{SAD}}\).

\vspace{-5pt}
\section{GPU-Based Parallelization of PADI}
\label{sec:gpu_padi}
\vspace{-5pt}

PADI requires repeated forward propagation through the frozen Deep SVDD encoder
when identifying the selective truncation region. This line-search step is
computationally expensive because, for many candidate values of the scalar
parameter \(z\), PADI must propagate both the current input \(X\) and its affine
representation \(A+Bz\) through the encoder and update the feasible interval of
\(z\). To make this procedure practical for deep encoders, we implement the forward and interval-update operations using custom Numba-CUDA kernels. Our implementation follows the GPU-accelerated SI strategy of
STAND-DA~\citep{kiet2026statistical}. Specifically, for fully connected
encoders, we reuse the shared-memory tiled matrix multiplication kernel
\texttt{MatMulMat} and adapt the \texttt{siReLU} idea to LeakyReLU, resulting in
the \texttt{SILeakyReLU} kernel. The only modification is that inactive units are
scaled by the LeakyReLU negative slope rather than being set to zero.
The main extension in PADI is the support for convolutional Deep SVDD encoders.
For CNN encoders, fully connected kernels alone are insufficient because the
encoder contains convolution, batch normalization, and max-pooling operations.
We therefore implement additional CUDA kernels for
\texttt{Conv2D}, \texttt{BatchNorm}, \texttt{SIMaxPool}, and the final fully
connected layer. These kernels propagate \(X\), \(A\), and \(B\) consistently
through the frozen encoder. Linear or affine layers, such as convolution,
batch normalization in inference mode, and fully connected layers, only transform
the affine representation. Piecewise-affine layers, such as LeakyReLU and max
pooling, additionally induce selection events and therefore update the feasible
interval.

The proposed GPU implementation reduces the computational cost of PADI for
convolutional Deep SVDD encoders in three ways. First, the main affine operations
in CNNs are parallelized at the feature-map level. Each convolutional output
element is computed by a GPU thread from the corresponding input channels and
kernel window, while batch normalization in inference mode is applied
element-wise using fixed channel-wise parameters. Second, the outputs of
piecewise-affine CNN layers are updated in parallel. LeakyReLU is applied
simultaneously across feature-map elements, and max-pooling outputs are computed
simultaneously across pooling windows. Third, local selection constraints are
computed together with these parallel layer updates. Each thread not only
computes its assigned LeakyReLU output or max-pooling output, but also derives
the local constraint on \(z\) required to preserve the activation branch
or pooling selection. These local constraints are merged on the GPU to update the feasible interval,
avoiding repeated CPU-side scans over all feature-map elements or pooling
windows. Fig.~\ref{fig:gpu_padi} illustrates the overall GPU-based
parallelization of PADI, while the detailed CUDA operations for convolutional
encoders are provided in Appendix~\ref{app:gpu_padi_details}.

\begin{figure}[t]
    \centering
    \includegraphics[width=\linewidth]{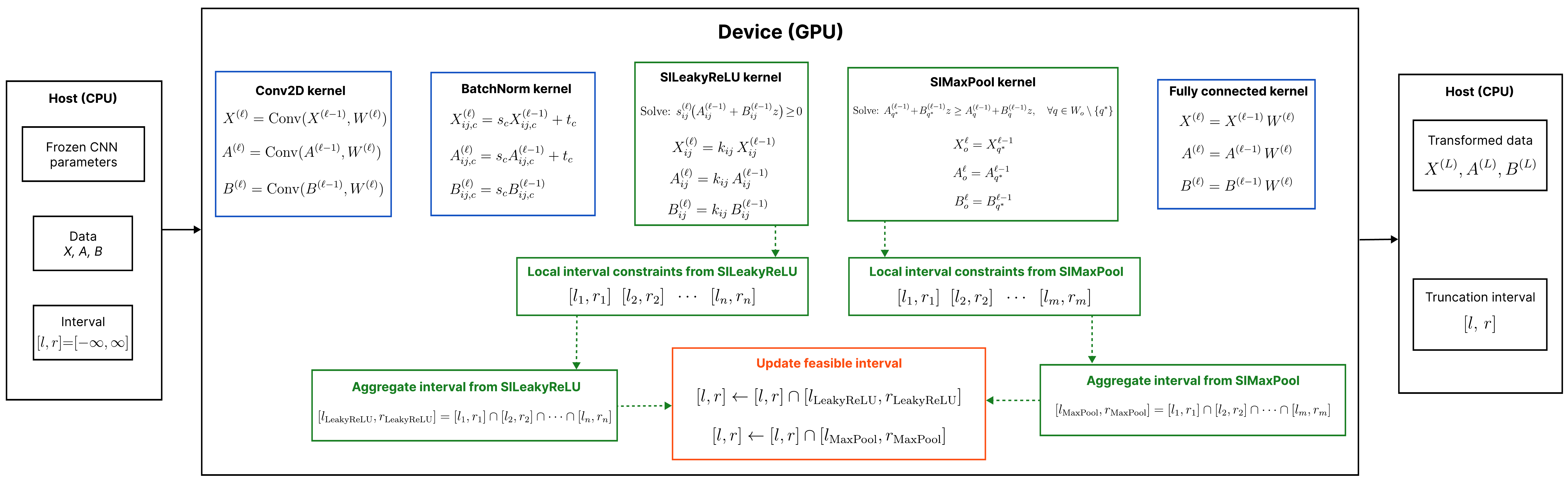}
    \caption{GPU-based parallelization of PADI for convolutional Deep SVDD encoders.
    Before the line-search computation, the frozen CNN parameters, the input data
    \(X,A,B\), and the current feasible interval are copied to GPU memory. The
    forward propagation is performed by custom Numba-CUDA kernels for
    \texttt{Conv2D}, \texttt{BatchNorm}, \texttt{SILeakyReLU}, \texttt{SIMaxPool},
    and the final fully connected layer. The \texttt{Conv2D}, \texttt{BatchNorm},
    and fully connected kernels propagate the forward values and affine coefficients
    \(X,A,B\). In addition, the \texttt{SILeakyReLU} and \texttt{SIMaxPool} kernels
    compute local interval constraints induced by the LeakyReLU branches
    and max-pooling indices. These local constraints are combined on the GPU to
    update the feasible interval. The final transformed data
    \(X^{(L)},A^{(L)},B^{(L)}\) and the final interval are then returned to the
    host.}
    \label{fig:gpu_padi}
    \vspace{-10pt}
\end{figure}

\vspace{-5pt}
\section{Experiments}
\label{sec:Experiments}

\vspace{-5pt}

In this section, we evaluate and compare the following methods:

- \textbf{PADI}: the proposed method for Deep SVDD.

- \textbf{OC}: 
an over-conditioning baseline that additionally conditions on the observed affine region of the frozen Deep SVDD encoder, with the truncation set constructed only from that region, analogous to $\mathcal{Z}_{\mathrm{OC}}$ described in Section~\ref{subsec:characterization_Z}.

- \textbf{Naive}: traditional statistical inference.

- \textbf{Op1}: an ablation study that excludes the sign-pattern constraint in Appendix~\ref{appsubsec:Z_sign}.

- \textbf{Op2}: another ablation study that excludes the anomaly detection event for Deep SVDD in Appendix~\ref{appsubsec:Z_score}.

A method that fails to control the FPR at the target significance level is regarded as statistically invalid, and its TPR is therefore not further considered. Throughout all experiments, we set the significance level to \(\alpha = 0.05\). The anomaly threshold $\tau$ for each trained Deep SVDD model is fixed to the empirical 95th percentile of the anomaly scores computed on the normal training data. This ensures that $\tau$ is determined independently of the test instances and reference samples used during the post-selection inference stage. The experiments for the extension to Deep SAD are provided in Appendix~\ref{app:dsad-experiments}.

\vspace{-5pt}
\subsection{Synthetic Data Experiments}
\label{subsec:synthetic}
\vspace{-5pt}

We conduct synthetic data experiments to evaluate both FPR control and TPR of the competing methods. We considered two types of covariance matrices: (i) \textbf{Independence:} \(\Sigma = I_d\), and  (ii) \textbf{Correlation:} \(\Sigma = \big[0.1^{|i-j|}\big]_{i,j} \in \mathbb{R}^{d \times d}.\)
The independence setting represents a simplified scenario where the noise
components across dimensions are uncorrelated. In contrast, the correlation
setting introduces dependencies among dimensions and provides a more challenging
scenario for evaluating the proposed inference procedure. Considering these two
settings allows us to examine whether the validity of PADI is maintained under
different covariance structures. In all synthetic experiments, the observed normal reference set used by the inference procedure is randomly sampled from an independent reference dataset generated from \(\mathcal{N}(\mathbf{0}_d, \Sigma)\), which is generated separately
from the test samples. For synthetic experiments, $\Sigma$ is known by construction and is directly used in the inference procedure. The neural network encoder used in these experiments has a three-layer architecture \([32,16,8]\) with Leaky ReLU activations, where the negative slope is set to \(0.01\). For the FPR experiment, we fix the data dimension at \(d=5\) and vary the sample size as \(n \in \{200,400,600,800\}\). For each value of \(n\), the data are generated from \(\mathcal{N}(\mathbf{0}_d, \Sigma)\). The experiment is repeated 1000 times to compute the empirical FPR at significance level \(\alpha=0.05\). For the TPR experiment, we fix \(d=5\) and \(n=100\), and consider \(\Delta \in \{1.5,2,2.5,3\}.\) For each value of \(\Delta\), we define \(\mu_{\Delta} = (\Delta,\ldots,\Delta)^\top \in \mathbb{R}^d\) and generate the data from \(\mathcal{N}(\mu_{\Delta},\Sigma)\). The experiment is repeated 1000 times to compute the empirical TPR. Since TPR is meaningful only for statistically valid methods, we report TPR results only for methods that successfully control the FPR in the preceding experiment.

The results are shown in Fig.~\ref{fig:synthetic-results}. In both the independent and correlated settings, PADI and OC successfully control the FPR around the significance level \(\alpha=0.05\) across all sample sizes. In contrast, the Naive approach produces substantially inflated FPR values in
both covariance settings. This behavior confirms that directly applying
classical hypothesis testing after anomaly detection leads to invalid inference
because the anomaly selection event is ignored. The ablation results further
highlight the necessity of each conditioning component in PADI. Op1, which
removes the sign-pattern conditioning, fails to maintain FPR control, showing
that this conditioning step is essential for handling the $\ell_1$ test
statistic. Similarly, Op2 fails when the anomaly-selection event is ignored,
confirming that accounting for the Deep SVDD selection mechanism is necessary
for valid post-selection inference. Therefore, Naive, Op1, and Op2 are excluded from the TPR comparison. In the TPR experiments, PADI consistently achieves higher TPR than OC in both independent and correlated settings, and the gap becomes more pronounced as \(\Delta\) increases. The higher TPR of PADI is consistent with the reduced conditioning discussed in Section ~\ref{subsec:characterization_Z}. OC restricts inference to the observed affine region, whereas PADI retains all affine regions compatible with the same anomaly-selection event. These results indicate that PADI achieves stronger statistical TPR than OC while maintaining valid FPR control.

\begin{figure*}[!t]
    \centering

    \begin{subfigure}{0.49\linewidth}
        \centering
        \includegraphics[width=.8\linewidth]{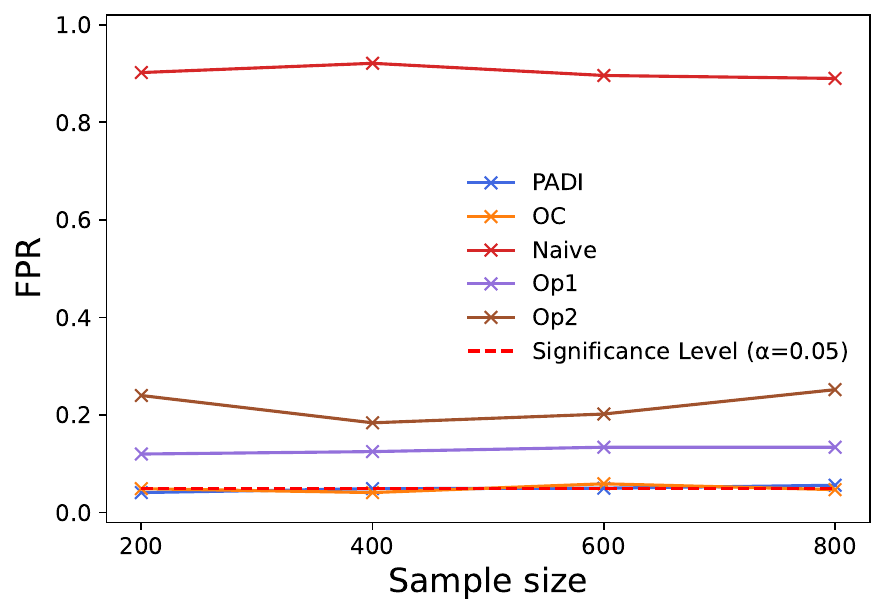}
        \caption{FPR on independent data}
        \label{fig:synthetic-fpr-independent}
    \end{subfigure}
    \hfill
    \begin{subfigure}{0.49\linewidth}
        \centering
        \includegraphics[width=.8\linewidth]{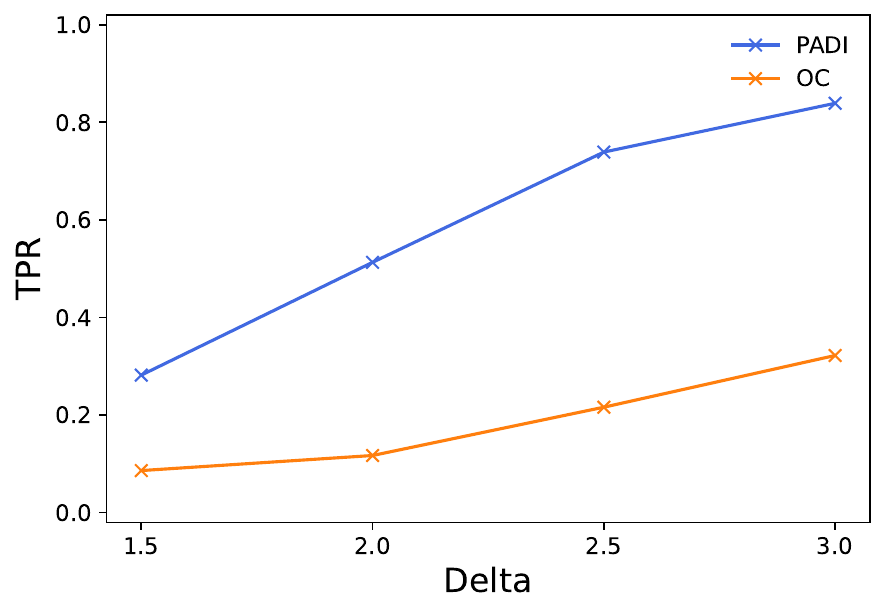}
        \caption{TPR on independent data}
        \label{fig:synthetic-tpr-independent}
    \end{subfigure}

    \vspace{5pt}

    \begin{subfigure}{0.49\linewidth}
        \centering
        \includegraphics[width=.8\linewidth]{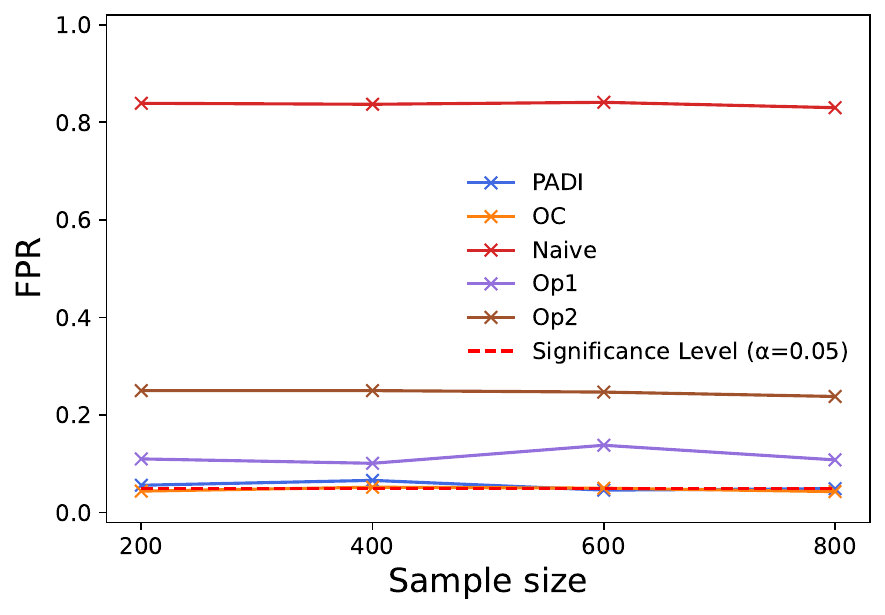}
        \caption{FPR on correlated data}
        \label{fig:synthetic-fpr-correlated}
    \end{subfigure}
    \hfill
    \begin{subfigure}{0.49\linewidth}
        \centering
        \includegraphics[width=.8\linewidth]{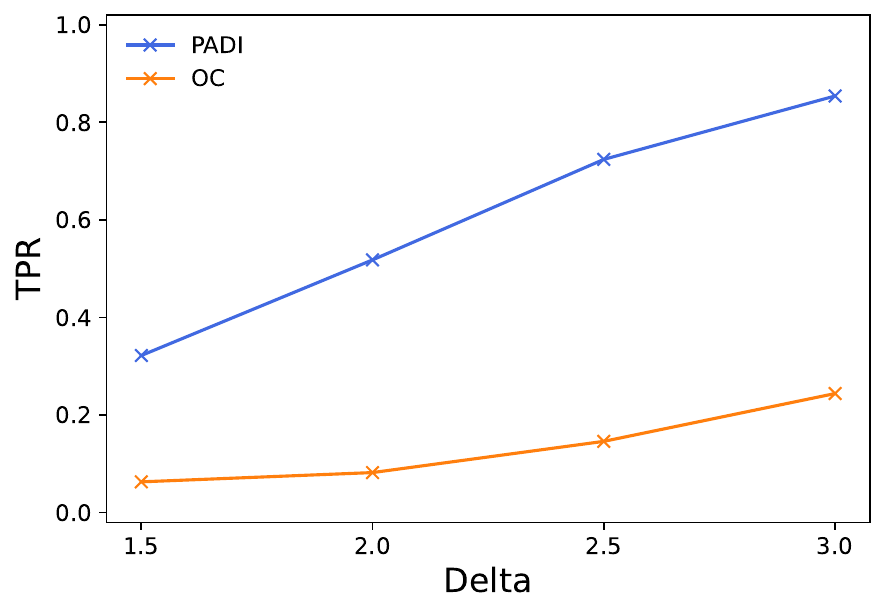}
        \caption{TPR on correlated data}
        \label{fig:synthetic-tpr-correlated}
    \end{subfigure}

    \caption{Results on synthetic data. }
    \label{fig:synthetic-results}
    \vspace{-15pt}
\end{figure*}

\vspace{5pt}
\textbf{Runtime evaluation of CNN-based GPU kernels.}
We further evaluate the efficiency of the proposed Numba-CUDA kernels by
comparing the runtime required to compute a selective $p$-value with a PyTorch-based
implementation. The purpose of this experiment is to assess whether the custom
CUDA kernels are beneficial for the repeated CNN forward propagation and
feasible-interval updates required by PADI. All runtime experiments are conducted
on an NVIDIA Tesla P100 GPU. We generate synthetic grayscale images of size
\(300\times 300\). Let
\(X\in\mathbb{R}^{300\times 300}\) denote an image. For normal images, each
pixel intensity is independently generated as
\[
    X_{h,w}\sim \mathcal{N}\left(\frac{255}{2},1\right),
    \quad h,w=1,\ldots,300.
\]
The generated pixel values are clipped to \([0,255]\) and normalized to \([0,1]\).
Each image is then divided into \(30\times30\) patches with stride \(30\).
Each patch is treated as an individual test instance.
The normal reference set used in the inference step is independently generated
from the same normal image distribution and processed using the same
patch-extraction procedure.

The convolutional encoder is composed of blocks of the form
\[
\mathrm{Conv2D}
\rightarrow
\mathrm{BatchNorm}
\rightarrow
\mathrm{LeakyReLU}
\rightarrow
\mathrm{MaxPool}.
\]
Each convolution uses a \(3\times3\) kernel with padding \(1\), and max pooling
uses a \(2\times2\) window with stride \(2\). 

We first evaluate the impact of encoder depth by varying the number of
convolutional blocks in \(\{4,5,6,7\}\). The four-block architecture uses channel
sizes \(16,32,64,128\), with max pooling applied only in the first block. The
deeper architectures are obtained by appending additional \(128\)-channel blocks
without max pooling. The results are shown in Fig.~\ref{fig:gpu_time_depth}. Across all
tested network depths, the Numba-CUDA implementation requires substantially less
computation time than the PyTorch-based implementation. Moreover, as the number
of convolutional blocks increases, the runtime gap becomes more pronounced.
These results indicate that the proposed custom CUDA kernels effectively
accelerate the CNN-based line-search computation in PADI, especially when the
frozen Deep SVDD encoder becomes deeper.

To further evaluate the scalability with respect to input dimensionality, we
conduct an additional experiment using the same experimental setting described
above. In this experiment, we fix the encoder architecture to the four-block
network and vary the input patch size from \(30\times30\), \(40\times40\),
\(50\times50\), to \(60\times60\), corresponding to input dimensions of
\(900\), \(1600\), \(2500\), and \(3600\), respectively. The results are shown
in Fig.~\ref{fig:gpu_time_dimension}, the runtime increases with
the input dimension for both implementations due to the increased number of
computations. However, the Numba-CUDA implementation consistently achieves lower runtime than the PyTorch-based implementation across all tested input dimensions. These results demonstrate that the proposed custom CUDA kernels remain effective in accelerating the CNN-based line-search computation in PADI as the input dimensionality increases.

\begin{figure}[tbp]
    \centering
    \begin{subfigure}{0.47\linewidth}
        \centering
        \includegraphics[width=\linewidth]{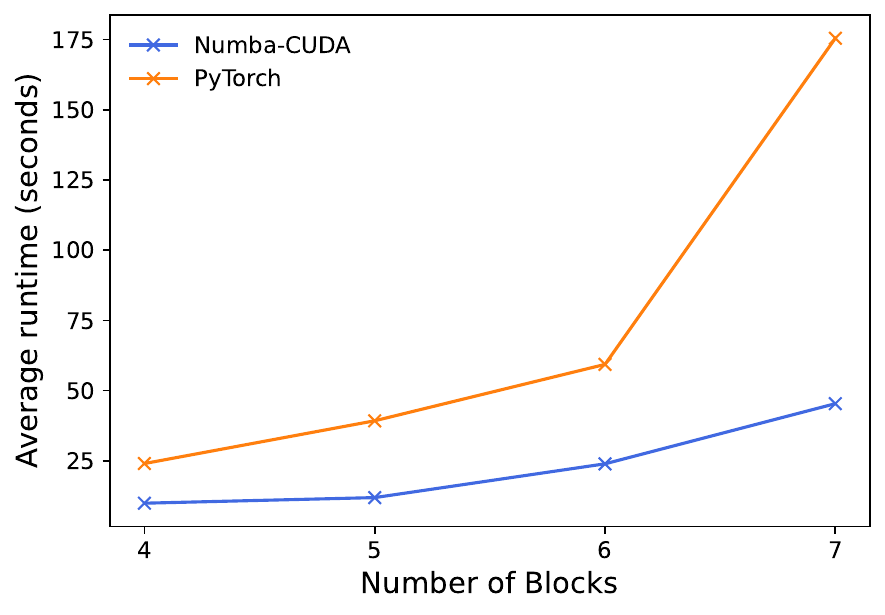}
        \caption{Different numbers of convolutional blocks.}
        \label{fig:gpu_time_depth}
    \end{subfigure}
    \hfill
    \begin{subfigure}{0.47\linewidth}
        \centering
        \includegraphics[width=\linewidth]{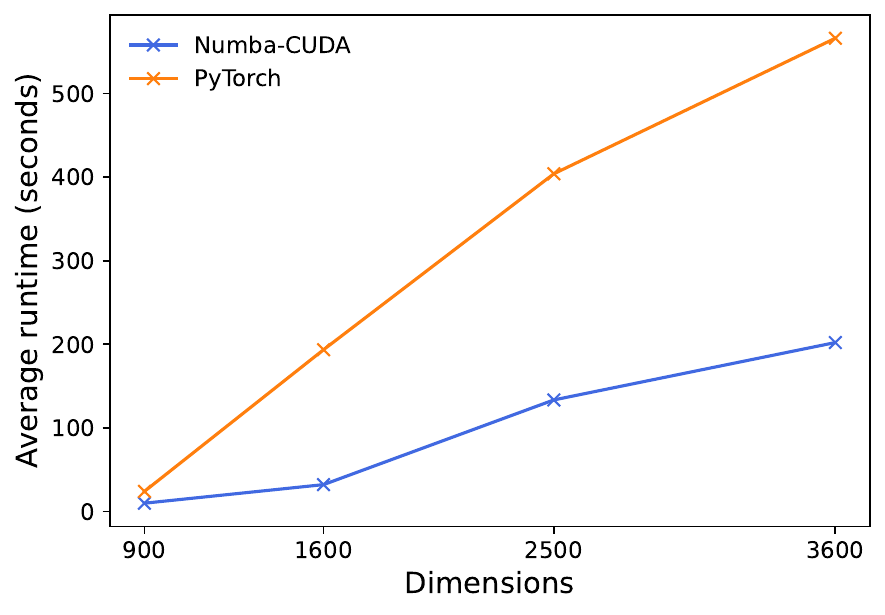}
        \caption{Different input dimensions.}
        \label{fig:gpu_time_dimension}
    \end{subfigure}
    
    \caption{Runtime comparison between the proposed Numba-CUDA implementation and
    the PyTorch-based implementation for CNN-based PADI.}
    \label{fig:gpu_time}
    \vspace{-15pt}
\end{figure}

\vspace{-5pt}
\subsection{Real-World Tabular Data Experiments}
\label{subsec:real_tabular}
\vspace{-5pt}

We evaluate the proposed method on 5 real-world tabular datasets: \textit{Breast Cancer}, \textit{Parkinson Disease}, \textit{Pharmacy Medicine}, \textit{Credit Fraud}, and \textit{Pulsar Stars}. These datasets cover different application domains and have feature dimensions ranging from 8 to 31. The neural network encoder used in these experiments has a seven-layer architecture \([128, 64, 32, 16, 8, 4, 2]\) with Leaky ReLU activations, where the negative slope is set to \(0.01\). In all experiments, we retain only numerical features and standardize each dataset so that every feature has zero mean and unit variance before applying anomaly detection and statistical inference. The normal reference set used in the inference step is randomly sampled from a
separate independent normal dataset that does not overlap with the test samples. The covariance matrix $\Sigma$ is estimated using a separate set of normal samples that is independent of both the test samples and the reference samples used for selective inference.

The empirical FPR and TPR results are shown in Fig.~\ref{fig:tabular-results}. Across all five datasets, PADI  and OC maintain the empirical FPR close to the target significance level, while the Naive approach exhibits severely inflated FPR values. The failure of the Naive approach is observed across datasets from different application domains, indicating that the selection bias caused by anomaly detection is an inherent issue of performing statistical inference after data-driven selection. Compared with OC, PADI achieves higher TPR on every evaluated dataset. This
consistent advantage across datasets with different feature dimensions and
application domains demonstrates that the proposed conditioning strategy can
retain higher statistical power in various tabular anomaly detection scenarios.


\begin{figure}[!t]
    \centering
    \includegraphics[width=.8\linewidth]{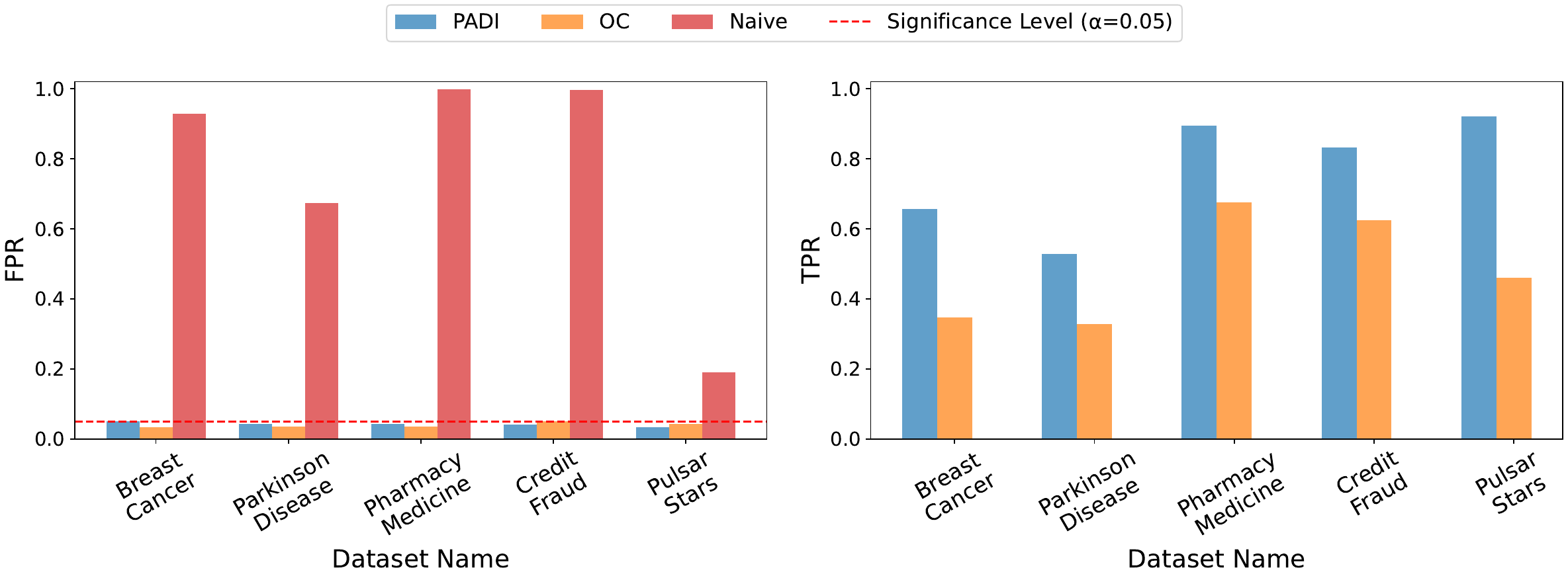}
    \caption{
    Results on real tabular datasets. Left: empirical FPR of PADI, OC, and Naive. Right: empirical TPR of the statistically valid methods. PADI achieves stronger TPR than OC while preserving FPR control.
    }
    \label{fig:tabular-results}
    \vspace{-5pt}
\end{figure}

\vspace{-5pt}
\subsection{Real-World Image Data Experiments}
\label{subsec:real_image}
\vspace{-5pt}


In this experiment, we use the MVTec~AD dataset
\citep{bergmann2019mvtec,bergmann2021mvtec} and evaluate five categories
chosen to span both texture subtypes described in MVTec~AD and to include
an object case: \textit{Carpet} and \textit{Grid} are the two regular-texture
categories, \textit{Tile} and \textit{Wood} are random-texture categories,
and \textit{Zipper} is an object category. This experiment is intended as a
proof-of-concept evaluation of PADI in a patch-level image setting rather
than as a comprehensive evaluation of all 15 MVTec~AD categories. All images are converted to grayscale, resized to \(300\times300\), and
converted to tensors in \([0,1]\). Each resized image is then divided into non-overlapping
\(30\times30\) patches with stride \(30\). The patch-level setting is motivated by the localized nature of many defects
in MVTec AD, which may occupy only a small fraction of an image
\citep{bergmann2021mvtec}. When the entire image is treated as a single test
instance, the contribution of a small defective region to the overall
discrepancy may be weak relative to the much larger normal region. This
motivation is consistent with \citet{huang2026unsupervised}, who show that
analyzing local patches increases the relative visibility and signal-to-noise
ratio of small defects. Accordingly, each $30\times30$ patch is treated as a
candidate test instance.
%
Patch-level anomaly labels are obtained from the pixel-level ground-truth
defect masks provided by MVTec AD. A patch is labeled anomalous if at least \(5\%\) of its pixels
overlap with the defect region; otherwise, it is treated as normal. This mask-based labeling avoids
incorrectly treating all patches from an anomalous image as anomalous, since defects in MVTec AD
are usually localized to a small region.  The convolutional encoder consists of two convolutional blocks of the form
$
\mathrm{Conv2D}
\rightarrow
\mathrm{BatchNorm}
\rightarrow
\mathrm{LeakyReLU}
\rightarrow
\mathrm{MaxPool},
$
with channel sizes \(8\) and \(16\). Each convolution uses a \(3\times3\) kernel
with padding \(1\), and each max-pooling layer uses a \(2\times2\) window with
stride \(2\). The final fully connected layer maps the resulting feature vector
to a \(16\)-dimensional latent representation. The normal reference set used in the inference step is randomly sampled from a
separate independent normal dataset that does not overlap with the test samples. The covariance matrix $\Sigma$ is estimated using a separate set of normal samples that is independent of both the test samples and the reference samples used for selective inference. 

The empirical FPR and TPR results are shown in Fig.~\ref{fig:image-results}. PADI and OC
successfully control the empirical FPR around the target significance level
across all five MVTec AD categories. In contrast, the Naive approach produces
substantially inflated FPR values across all categories, demonstrating that
conventional statistical testing becomes unreliable when applied after
patch-level anomaly selection. Regarding detection power, PADI consistently achieves higher TPR than OC across
all evaluated image categories. These results show that PADI maintains its advantage across different types of
image anomalies evaluated in the experiments.

\begin{figure}[!t]
    \centering
    \includegraphics[width=.8\linewidth]{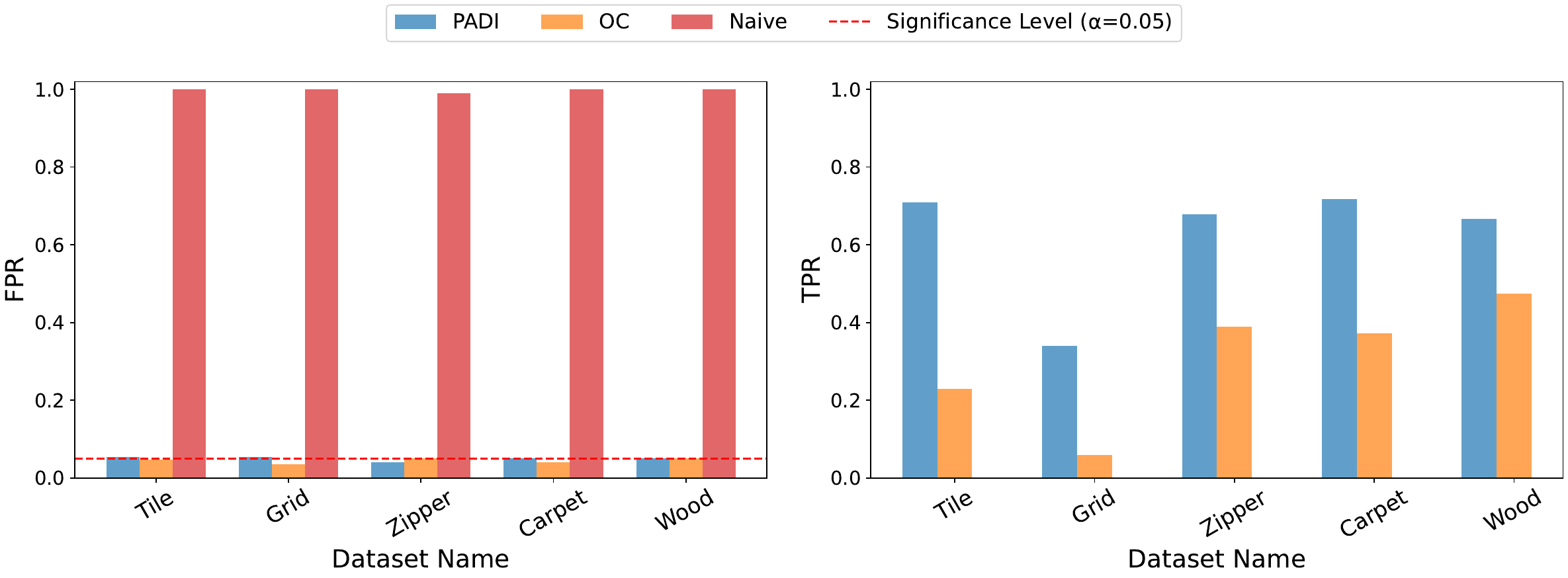}
    \caption{
    Results on real image data from the MVTec AD dataset. Left: empirical FPR of PADI, OC, and Naive across five classes.
    Right: empirical TPR of the statistically valid methods.
    PADI maintains FPR control while achieving higher TPR
    than OC.
    }
    \label{fig:image-results}
    \vspace{-10pt}
\end{figure}

\section{Limitations}
\label{sec:limitations}

PADI's current exact derivation requires the frozen encoder to be piecewise
affine, as stated in Assumption~1. Standard ReLU-based ResNets satisfy this
requirement: convolutional and linear layers, inference-mode Batch
Normalization, and average pooling are affine, whereas ReLU and max pooling
are piecewise affine. Residual connections preserve this property, and
stacking such residual blocks preserves it as well. In contrast, GELU,
Layer Normalization, and standard self-attention are not piecewise affine.
Architectures containing these operations, including standard ViTs, are
therefore not directly covered by the current derivation. Extending PADI's
exact truncation-region construction to such architectures is left for
future work.

Beyond this architectural limitation, PADI's selective \(p\)-values should
not be interpreted as unconditional guarantees of detector reliability.
Their statistical validity relies on the fixed-detector setting in
Section~\ref{subsec:frozen_detector_problem} and the statistical model and
associated assumptions in
Section~\ref{subsec:hypothesis_testing_problem}. If these assumptions are
violated, the \(p\)-values may be miscalibrated and the stated FPR-control
guarantee may no longer hold. Accordingly, PADI should not be deployed in
high-stakes settings without checking these assumptions and conducting
domain-specific validation.

In particular, the current theoretical guarantee relies on the Gaussian
test-reference model, which is used to derive the conditional null distribution
of the selective $p$-value. Therefore, PADI provides exact finite-sample
validity under the stated model assumptions rather than distribution-free
validity. The assumptions regarding independent normal reference samples,
fixed anomaly-selection thresholds, and known or independently estimated
covariance matrices are required to ensure that the selective inference
procedure is correctly calibrated. Extending the theoretical guarantee to more
general data distributions or relaxing these assumptions remains an important
direction for future work.

Another limitation concerns the interpretation of the selective $p$-value, which
depends on the chosen test statistic. In the current formulation, PADI
quantifies the statistical significance of the input-space $\ell_1$
discrepancy between a selected test sample and the mean of independent normal
reference samples, conditional on the Deep SVDD selection event. Therefore, the
resulting $p$-value should not be interpreted as a universal probability of
anomaly or a measure of semantic abnormality. For anomaly detection tasks where
abnormality is defined by high-level semantic differences rather than
input-space deviations, the current test statistic may not capture the desired
notion of abnormality, and an alternative inferential target or test statistic
may be required. Developing selective inference procedures with more general
test statistics for such broader notions of anomaly remains an important
direction for future work.

\vspace{-5pt}
\section{Conclusion}
\vspace{-5pt}

In this paper, we proposed PADI, a statistically rigorous inference framework for Deep SVDD-based AD that associates detected anomalies with valid $p$-values. By formulating anomaly assessment within the SI framework, PADI addresses the double-dipping issue inherent in post-selection inference and provides theoretical guarantees for controlling the FPR at a user-specified significance level. The proposed method operates in a post hoc manner and can be seamlessly applied to trained and frozen Deep SVDD models without requiring retraining, while also extending naturally to deep semi-supervised anomaly detection. Extensive experiments on synthetic and real-world datasets demonstrate that PADI consistently achieves reliable FPR control while maintaining strong detection power. These results establish PADI as a practical and statistically sound framework for enhancing the reliability of 
Deep SVDD-based AD systems.

\subsection*{Acknowledgements}
This research was supported by The VNUHCM-University of Information Technology’s Scientific Research Support Fund.

\FloatBarrier
\bibliography{main}
\bibliographystyle{tmlr}

\newpage

\appendix
\section*{Appendix}

\section{Proofs of the main theorems}
\label{app:proofs}

The proofs below instantiate the Gaussian-conditioning framework of \citet{lee2016exact}. The inherited steps are the conditional truncated-Gaussian validity argument and the one-dimensional affine-line reduction; relative to Lee et al., the application-specific step is the construction of the PADI truncation region $\mathcal{Z}$, detailed in Appendix~\ref{app:truncation_region}.

\subsection{Proof of Theorem~\ref{thm:validity_main}}
\label{app:proof_validity}

Let
\[
\mathcal O(\bm Y)
=
\left(
\cA(\vX^{\mathrm{test}}),
\cS(\bm Y)
\right).
\]
Under \(H_0\), we have
\[
\veta^\top\bm Y
\mid
\left\{
\mathcal O(\bm Y)=\mathcal O(\bm y),
\cQ(\bm Y)=\cQ(\bm y)
\right\}
\sim
{\rm TN}
\left(
0,\,
\veta^\top\tilde{\mSigma}\veta,\,
\mathcal Z
\right),
\]
which is a truncated normal distribution with mean \(0\), variance
\(\veta^\top\tilde{\mSigma}\veta\), and truncation region \(\mathcal Z\). 
Because $\mathcal{Z}$ is a finite union of intervals, applying the probability integral transform to the conditional distribution of $|\veta^\top\bm Y|$, using the argument of Theorem~5.2 and the corresponding union-of-intervals argument in Theorem~5.3 of \citet{lee2016exact}, yields the stated uniform distribution of $p^{\mathrm{selective}}$.

Therefore, under the  null hypothesis,
\[
p^{\mathrm{selective}}
\mid
\left\{
\mathcal O(\bm Y)=\mathcal O(\bm y),
\cQ(\bm Y)=\cQ(\bm y)
\right\}
\sim
\mathrm{Unif}(0,1).
\]
Thus,
\[
\mathbb P_{H_0}
\left(
p^{\mathrm{selective}}\le \alpha
\;\middle|\;
\mathcal O(\bm Y)=\mathcal O(\bm y),
\cQ(\bm Y)=\cQ(\bm y)
\right)
=
\alpha,
\qquad
\forall \alpha\in[0,1].
\]

Next, we have
\[
\begin{aligned}
&
\mathbb P_{H_0}
\left(
p^{\mathrm{selective}}\le \alpha
\;\middle|\;
\mathcal O(\bm Y)=\mathcal O(\bm y)
\right)
\\
&=
\int
\mathbb P_{H_0}
\left(
p^{\mathrm{selective}}\le \alpha
\;\middle|\;
\mathcal O(\bm Y)=\mathcal O(\bm y),
\cQ(\bm Y)=\cQ(\bm y)
\right)
\mathbb P_{H_0}
\left(
\cQ(\bm Y)=\cQ(\bm y)
\;\middle|\;
\mathcal O(\bm Y)=\mathcal O(\bm y)
\right)
\,d\cQ(\bm y)
\\
&=
\int
\alpha
\mathbb P_{H_0}
\left(
\cQ(\bm Y)=\cQ(\bm y)
\;\middle|\;
\mathcal O(\bm Y)=\mathcal O(\bm y)
\right)
\,d\cQ(\bm y)
\\
&=
\alpha
\int
\mathbb P_{H_0}
\left(
\cQ(\bm Y)=\cQ(\bm y)
\;\middle|\;
\mathcal O(\bm Y)=\mathcal O(\bm y)
\right)
\,d\cQ(\bm y)
\\
&=
\alpha.
\end{aligned}
\]

Finally, averaging over all possible realizations of the selection event, we
obtain
\[
\begin{aligned}
\mathbb P_{H_0}
\left(
p^{\mathrm{selective}}\le \alpha
\right)
&=
\sum_{\mathcal O(\bm y)}
\mathbb P_{H_0}
\left(
p^{\mathrm{selective}}\le \alpha
\;\middle|\;
\mathcal O(\bm Y)=\mathcal O(\bm y)
\right)
\mathbb P_{H_0}
\left(
\mathcal O(\bm Y)=\mathcal O(\bm y)
\right)
\\
&=
\sum_{\mathcal O(\bm y)}
\alpha
\mathbb P_{H_0}
\left(
\mathcal O(\bm Y)=\mathcal O(\bm y)
\right)
=
\alpha.
\end{aligned}
\]
This proves Theorem~\ref{thm:validity_main}.

\subsection{Proof of Theorem~\ref{thm:data_line}}
\label{app:proof_data_line}

Conditional on $\cS(\bm Y)=\cS(\bm y)$, the contrast direction $\veta$ is fixed. Under the Gaussian model, Eqs.~(5.2)--(5.3) of \citet{lee2016exact} imply
\[
y=z_{\mathrm{Lee}}+c(\eta^\top y),
\qquad
z_{\mathrm{Lee}}=(I_n-c\eta^\top)y,
\]
where $z_{\mathrm{Lee}}$ is independent of $\eta^\top y$. Under the correspondence $y\leftrightarrow\bm Y$, $\Sigma\leftrightarrow\tilde{\mSigma}$, and $c\leftrightarrow\vb$, their vector $z_{\mathrm{Lee}}$ corresponds to $\cQ(\bm Y)$, whereas our scalar $z=\veta^\top\bm Y$ corresponds to their contrast $\eta^\top y$. The short derivation below is included only to establish the notation used to characterize $\mathcal{Z}$.

Recall that
\[
z=\veta^\top \bm Y,
\qquad
\cQ(\bm Y)
=
\left(\mI_{(m+1)D}-\vb\veta^\top\right)\bm Y
=
\bm Y-\vb z,
\]
where
\[
\vb
=
\frac{\tilde{\mSigma}\veta}{\veta^\top\tilde{\mSigma}\veta}.
\]
Let
\[
\va=\cQ(\bm y),
\qquad
z_{\rm obs}=\veta^\top\bm y.
\]
Since \(\cQ(\bm Y)=\bm Y-\vb z\), the condition
\(\cQ(\bm Y)=\cQ(\bm y)\) implies
\[
\bm Y-\vb z=\va,
\]
or equivalently,
\[
\bm Y=\va+\vb z.
\]
Hence, after conditioning on \(\cQ(\bm Y)=\cQ(\bm y)\), the remaining
randomness is indexed only by the scalar \(z\) along the one-dimensional affine
line
\[
\bm Y(z)=\va+\vb z,
\qquad z\in\mathbb R.
\]

The remaining non-nuisance conditions restrict \(z\) to the set
\[
\mathcal Z
=
\left\{
z\in\mathbb R
\;\middle|\;
\begin{array}{l}
\cA(\vX^{\rm test}(z))=\cA(\bm x^{\rm test}),\\
\cS(\bm Y(z))=\cS(\bm y)
\end{array}
\right\}.
\]
Thus, the conditioning set \(\mathcal D\) can be written as
\[
\mathcal D
=
\left\{
\bm Y(z)=\va+\vb z
\;\middle|\;
z\in\mathcal Z
\right\},
\]
which proves Theorem~\ref{thm:data_line}.

\section{Characterization of the Truncation Region}
\label{app:truncation_region}

In this appendix, we provide the detailed derivations for the truncation
region \(\mathcal Z\) summarized in Section~\ref{subsec:characterization_Z}.
As described in the main text, the truncation region is decomposed as
\[
\mathcal Z = \mathcal Z_{\mathrm{sign}} \cap \mathcal Z_{\mathrm{AD}}.
\]
We first decompose the affine-line representation
\(\bm Y(z)=\va+\vb z\) into its test and reference components, and then
derive the two sub-problems, \(\mathcal Z_{\mathrm{sign}}\) and
\(\mathcal Z_{\mathrm{AD}}\), in detail.

By Theorem~\ref{thm:data_line}, conditioning on
\(\cQ(\bm Y)=\cQ(\bm y)\) restricts the random vector \(\bm Y\) to the
one-dimensional affine line
\[
\bm Y(z)=\va+\vb z,
\qquad z\in\mathbb R,
\]
where \(\va,\vb\in\mathbb R^{(m+1)D}\). Since \(\bm Y\) is stacked in the order
\[
\bm Y=
\operatorname{vec}
\left(
\vX^{\mathrm{test}},
\vX^{1,\mathrm{ref}},
\dots,
\vX^{m,\mathrm{ref}}
\right),
\]
we partition \(\va\) and \(\vb\) conformably as
\[
\va
=
\operatorname{vec}
\left(
\va^{\mathrm{test}},
\va^{1,\mathrm{ref}},
\dots,
\va^{m,\mathrm{ref}}
\right),
\qquad
\vb
=
\operatorname{vec}
\left(
\vb^{\mathrm{test}},
\vb^{1,\mathrm{ref}},
\dots,
\vb^{m,\mathrm{ref}}
\right),
\]
where
\[
\va^{\mathrm{test}},\vb^{\mathrm{test}}\in\mathbb R^D,
\qquad
\va^{j,\mathrm{ref}},\vb^{j,\mathrm{ref}}\in\mathbb R^D,
\quad j=1,\dots,m.
\]
Accordingly, each point on the affine line can be written as
\[
\bm Y(z)
=
\operatorname{vec}
\left(
\vX^{\mathrm{test}}(z),
\vX^{1,\mathrm{ref}}(z),
\dots,
\vX^{m,\mathrm{ref}}(z)
\right),
\]
where
\[
\vX^{\mathrm{test}}(z)
=
\va^{\mathrm{test}}+\vb^{\mathrm{test}}z
\in\mathbb R^D,
\]
and, for each \(j=1,\dots,m\),
\[
\vX^{j,\mathrm{ref}}(z)
=
\va^{j,\mathrm{ref}}+\vb^{j,\mathrm{ref}}z
\in\mathbb R^D.
\]

\subsection{Sign-pattern constraint}
\label{appsubsec:Z_sign}

The sign-pattern constraint keeps fixed the signs used to linearize the
\(\ell_1\)-discrepancy. Define the reference mean along the affine line by
\[
{\bar{\vX}}^{\mathrm{ref}}(z)
=
\frac{1}{m}
\sum_{j=1}^m
\vX^{j, \mathrm{ref}}(z).
\]
Using the block representation above, this can be written as
\[
{\bar{\vX}}^{\mathrm{ref}}(z)
=
{\bar{\va}}^{\mathrm{ref}}
+
{\bar{\vb}}^{\mathrm{ref}}z,
\]
where
\[
{\bar{\va}}^{\mathrm{ref}}
=
\frac{1}{m}
\sum_{j=1}^m
\va^{j, \mathrm{ref}},
\qquad
{\bar{\vb}}^{\mathrm{ref}}
=
\frac{1}{m}
\sum_{j=1}^m
\vb^{j, \mathrm{ref}}.
\]

For each coordinate \(u=1,\dots,D\), we have
\[
x_u^{\mathrm{test}}(z)-\bar{x}_u^{\mathrm{ref}}(z)
=
\left(
a_u^{\mathrm{test}}-\bar a_u^{\mathrm{ref}}
\right)
+
\left(
b_u^{\mathrm{test}}-\bar b_u^{\mathrm{ref}}
\right)z.
\]
Let
\[
\alpha_u
=
a_u^{\mathrm{test}}-\bar a_u^{\mathrm{ref}},
\qquad
\gamma_u
=
b_u^{\mathrm{test}}-\bar b_u^{\mathrm{ref}}.
\]
Then the sign event
\[
\cS(\bm Y(z))=\cS(\bm y)
\]
is equivalent to
\[
\cS_u(\bm y)
\left(
\alpha_u+\gamma_u z
\right)
>0,
\qquad
u=1,\dots,D.
\]
Thus, the sign-pattern feasible set is
\[
\mathcal Z_{\mathrm{sign}}
=
\left\{
z\in\mathbb R
\;\middle|\;
\cS_u(\bm y)
\left(
\alpha_u+\gamma_u z
\right)
>0,
\quad
u=1,\dots,D
\right\}.
\]
Each constraint is a linear inequality in the scalar variable \(z\).

\subsection{Affine-region characterization of the frozen encoder}
\label{appsubsec:Z_region}

We next describe how the affine regions of the frozen encoder are used to
characterize the Deep SVDD selection event. Recall from the decomposition
above that the test sample varies along
\(\vX^{\mathrm{test}}(z)=\va^{\mathrm{test}}+\vb^{\mathrm{test}}z\),
and let \(\mathfrak P_{\mathrm{line}}\) denote the collection of affine
regions intersected by this path (as introduced in Lemma~\ref{lem:z_ad}). Each such
region \(\mathcal P\in\mathfrak P_{\mathrm{line}}\) is a  polytope
defined by \(M_{\mathcal P}\) linear inequalities: there exist a matrix
\(\mG_{\mathcal P}\in\mathbb R^{M_{\mathcal P}\times D}\) and a vector
\(\vh_{\mathcal P}\in\mathbb R^{M_{\mathcal P}}\) such that
\[
\mathcal P
=
\left\{
\vx\in\mathbb R^D:
\mG_{\mathcal P}\vx\le \vh_{\mathcal P}
\right\}.
\]
The set of \(z\)-values for which the test sample lies in the affine region
\(\mathcal P\) is
\[
\mathcal Z_{\mathrm{region}}(\mathcal P)
=
\left\{
z\in\mathbb R
\;\middle|\;
\mG_{\mathcal P}
\left(
\va^{\mathrm{test}}+\vb^{\mathrm{test}}z
\right)
\le
\vh_{\mathcal P}
\right\}.
\]
Each \(\mathcal Z_{\mathrm{region}}(\mathcal P)\) is characterized by finitely
many linear inequalities in \(z\).

\subsection{Deep SVDD selection constraint}
\label{appsubsec:Z_score}

The Deep SVDD detector selects the test sample as anomalous when
$\cA(\vX^{\mathrm{test}}(z))=\cA(\bm x^{\mathrm{test}})$, which by the definition of the anomaly score
in \eqref{eq:score_general} is equivalent to
$g(\vX^{\mathrm{test}}(z))\ge \tau$.

By Assumption~\ref{assump:pwa}, within each affine region
\(\mathcal P\in\mathfrak P_{\mathrm{line}}\), the encoder acts as
\(\hat\phi(\vx)=\mL_{\mathcal P}\vx+\vbeta_{\mathcal P}\).
Substituting the test-sample path
\(\vX^{\mathrm{test}}(z)=\va^{\mathrm{test}}+\vb^{\mathrm{test}}z\)
from the decomposition at the beginning of this appendix, we obtain
the following. For \(z\in\mathcal Z_{\mathrm{region}}(\mathcal P)\), the point
\(\vX^{\mathrm{test}}(z)\) lies in region \(\mathcal P\), so the affine
representation of the encoder on \(\mathcal P\) gives
\[
\hat\phi(\vX^{\mathrm{test}}(z))-\hat{\vc}
=
\mL_{\mathcal P}
\left(
\va^{\mathrm{test}}+\vb^{\mathrm{test}}z
\right)
+
\vbeta_{\mathcal P}
-
\hat{\vc}.
\]
Separating the constant term and the coefficient of \(z\), define
\[
\vu_0(\mathcal P)
=
\mL_{\mathcal P}\va^{\mathrm{test}}
+
\vbeta_{\mathcal P}
-
\hat{\vc},
\qquad
\vu_1(\mathcal P)
=
\mL_{\mathcal P}\vb^{\mathrm{test}}.
\]
Then
\[
\hat\phi(\vX^{\mathrm{test}}(z))-\hat{\vc}
=
\vu_0(\mathcal P)+\vu_1(\mathcal P)z.
\]

Therefore, within affine region \(\mathcal P\), the Deep SVDD score along the
path is
\[
g_{\mathcal P}(z)
=
\left\|
\vu_0(\mathcal P)+\vu_1(\mathcal P)z
\right\|_2^2.
\]
Expanding the squared norm gives a quadratic function of the scalar variable
\(z\):
\[
g_{\mathcal P}(z)
=
\kappa_2(\mathcal P)z^2
+
\kappa_1(\mathcal P)z
+
\kappa_0(\mathcal P),
\]
where
\[
\kappa_2(\mathcal P)
=
\|\vu_1(\mathcal P)\|_2^2,
\qquad
\kappa_1(\mathcal P)
=
2\vu_0(\mathcal P)^\top\vu_1(\mathcal P),
\qquad
\kappa_0(\mathcal P)
=
\|\vu_0(\mathcal P)\|_2^2.
\]

Consequently, within region \(\mathcal P\), the set of values of \(z\) for which
the test sample is selected as anomalous by Deep SVDD is
\[
\mathcal Z_{\mathrm{score}}(\mathcal P)
=
\left\{
z\in\mathbb R
\;\middle|\;
\kappa_2(\mathcal P)z^2+
\kappa_1(\mathcal P)z+
\kappa_0(\mathcal P)
\ge \tau
\right\}.
\]
Thus, on each affine region \(\mathcal P\), the Deep SVDD selection constraint
reduces to a quadratic inequality in the one-dimensional variable \(z\).

\section{Details of GPU-Based Parallelization of PADI}
\label{app:gpu_padi_details}

This appendix provides implementation details for the GPU-based parallelization
of PADI described in Section~\ref{sec:gpu_padi}. The computational bottleneck of
PADI comes from the line-search procedure used to identify the selective
truncation region. During this procedure, the frozen Deep SVDD encoder is
repeatedly evaluated along the scalar parameter \(z\) of the selective-inference
line. At each evaluation, PADI propagates the ordinary forward values and their
affine coefficients with respect to \(z\), while updating the feasible interval
whenever a layer induces a data-dependent selection constraint.

For fully connected encoders, PADI follows the GPU-based selective-inference
implementation strategy of STAND-DA~\citep{kiet2026statistical}, which provides
CUDA kernels for matrix transformations and ReLU activation-pattern
conditioning. We therefore do not repeat the fully connected implementation
details here. The main GPU extension in PADI is the support for convolutional
Deep SVDD encoders, which require additional kernels for convolution, batch
normalization, LeakyReLU, max pooling, and the final fully connected layer.

Let \(L\) denote the total number of layers in the frozen encoder. Here, each
operation is counted as one layer; for example, \texttt{Conv2D},
\texttt{BatchNorm}, \texttt{LeakyReLU}, \texttt{MaxPool}, and
\texttt{FullyConnected} are all treated as separate layers. We use
\(\ell=1,\ldots,L\) to index these layers. At layer \(\ell\), let
\(X^{(\ell)}\) denote the ordinary forward output,
and let \(A^{(\ell)}\) and \(B^{(\ell)}\) denote the affine coefficients of the
same layer output with respect to \(z\). Thus, along the selective-inference
line, the layer output is represented by
\[
    X^{(\ell)}(z)=A^{(\ell)}+B^{(\ell)}z,
\]

\paragraph{Conv2D kernel.}
For a convolutional layer \(\ell\), the \texttt{Conv2D} kernel assigns one CUDA
thread to one output feature-map element. Each thread computes the convolution
over the corresponding input channels and spatial kernel window using the fixed
filter \(W^{(\ell)}\). Since convolution is linear in its input, the ordinary
forward values and affine coefficients are propagated as
\[
    X^{(\ell)}
    =
    \operatorname{Conv}\!\left(X^{(\ell-1)},W^{(\ell)}\right),
\]
\[
    A^{(\ell)}
    =
    \operatorname{Conv}\!\left(A^{(\ell-1)},W^{(\ell)}\right),
\]
\[
    B^{(\ell)}
    =
    \operatorname{Conv}\!\left(B^{(\ell-1)},W^{(\ell)}\right).
\]
In our implementation, convolutional layers are bias-free. If a convolutional
bias is used, it should be added to \(X^{(\ell)}\) and \(A^{(\ell)}\), but not to
\(B^{(\ell)}\), since the bias is independent of \(z\). The convolutional kernel
does not introduce any selection constraint; it only propagates \(X\), \(A\), and
\(B\).

\paragraph{BatchNorm kernel.}
For a batch normalization layer \(\ell\), the layer is evaluated in inference
mode. Hence, the running statistics and learned affine parameters are fixed. For
channel \(c\), define
\[
    s_c^{(\ell)}
    =
    \frac{\gamma_c^{(\ell)}}{
    \sqrt{(\sigma_c^{(\ell)})^2+\epsilon}},
    \qquad
    t_c^{(\ell)}
    =
    \beta_c^{(\ell)}-s_c^{(\ell)}\mu_c^{(\ell)},
\]
where \(\mu_c^{(\ell)}\) and \((\sigma_c^{(\ell)})^2\) are the frozen running
mean and variance, and \(\gamma_c^{(\ell)}\) and \(\beta_c^{(\ell)}\) are the
learned scale and shift parameters. For each tensor element \((i,j,c)\), the
CUDA kernel applies the channel-wise affine transformation
\[
    X_{ij, c}^{(\ell)}
    =
    s_c^{(\ell)}X_{ij, c}^{(\ell-1)}+t_c^{(\ell)},
\]
\[
    A_{ij, c}^{(\ell)}
    =
    s_c^{(\ell)}A_{ij, c}^{(\ell-1)}+t_c^{(\ell)},
\]
\[
    B_{ij, c}^{(\ell)}
    =
    s_c^{(\ell)}B_{ij, c}^{(\ell-1)}.
\]
The shift term affects the ordinary forward value \(X\) and the intercept
coefficient \(A\), but not the slope coefficient \(B\). Since batch normalization
in inference mode is affine, it does not add any new constraint on \(z\). The
kernel is parallelized element-wise, with one CUDA thread processing one
feature-map element.

\paragraph{SILeakyReLU kernel.}
For a LeakyReLU layer \(\ell\), each feature-map element is processed
independently by one CUDA thread. The role of this kernel is twofold: it preserves
the observed LeakyReLU branch along the selective-inference line, and it
propagates the ordinary forward value and affine coefficients through the fixed
branch.

For each feature-map element, define the branch sign
\[
    s_{ij}^{(\ell)}
    =
    \begin{cases}
    1, & X_{ij}^{(\ell-1)}\ge 0,\\
    -1, & X_{ij}^{(\ell-1)}<0.
    \end{cases}
\]
Preserving the LeakyReLU branch is equivalent to imposing the linear inequality
\[
    s_{ij}^{(\ell)}
    \left(
        A_{ij}^{(\ell-1)}
        +
        B_{ij}^{(\ell-1)}z
    \right)
    \ge 0 .
\]
Each
feature-map element therefore contributes one local linear constraint on \(z\).
In the CUDA implementation, the thread assigned to element \((i,j)\) converts
this inequality into a local lower-bound or upper-bound candidate for the
feasible interval. The local candidates generated by all threads are merged on
the GPU to update the feasible interval.

After the branch constraint is computed, the kernel propagates \(X\), \(A\), and
\(B\) through the selected LeakyReLU branch. Let \(\rho\) denote the negative
slope of LeakyReLU, and define the branch slope
\[
    k_{ij}^{(\ell)}
    =
    \begin{cases}
    1, & X_{ij}^{(\ell-1)}\ge 0,\\
    \rho, & X_{ij}^{(\ell-1)}<0.
    \end{cases}
\]
Then the propagated quantities are
\[
    X_{ij}^{(\ell)}
    =
    k_{ij}^{(\ell)}X_{ij}^{(\ell-1)},
\]
\[
    A_{ij}^{(\ell)}
    =
    k_{ij}^{(\ell)}A_{ij}^{(\ell-1)},
\]
\[
    B_{ij}^{(\ell)}
    =
    k_{ij}^{(\ell)}B_{ij}^{(\ell-1)}.
\]
The \texttt{SILeakyReLU} kernel extends the ReLU-based selective activation
kernel of STAND-DA~\citep{kiet2026statistical} to LeakyReLU: inactive elements
are not set to zero, but are multiplied by the negative slope \(\rho\).

\paragraph{SIMaxPool kernel.}
For a max-pooling layer \(\ell\), each output element copies the maximum value
from a pooling window. Let \(o=(i,j)\) denote an output position and let
\(\mathcal{W}_o\) be the set of input locations in the pooling window associated
with \(o\). The maximum index is
\[
    q^\star
    =
    \arg\max_{q\in\mathcal{W}_o}
    X_q^{(\ell-1)}.
\]
To preserve the same max-pooling selection along the line, the selected location
must remain no smaller than every other location in the same window:
\[
    A_{q^\star}^{(\ell-1)}
    +
    B_{q^\star}^{(\ell-1)}z
    \ge
    A_q^{(\ell-1)}
    +
    B_q^{(\ell-1)}z,
    \qquad
    \forall q\in\mathcal{W}_o\setminus\{q^\star\}.
\]
Equivalently,
\[
    \left(
    B_{q^\star}^{(\ell-1)}-B_q^{(\ell-1)}
    \right)z
    \ge
    A_q^{(\ell-1)}-A_{q^\star}^{(\ell-1)},
    \qquad
    \forall q\in\mathcal{W}_o\setminus\{q^\star\}.
\]
Thus, one max-pooling output element can generate multiple local linear
constraints on \(z\), one for each comparison inside the pooling window.

The \texttt{SIMaxPool} kernel assigns one CUDA thread to one output pooling
position. The thread identifies \(q^\star\), compares the selected location
with the remaining locations in the same pooling window, and converts the
resulting inequalities into local lower-bound or upper-bound candidates. These
local candidates are merged on the GPU to update the feasible interval.

After the selected maximum index is fixed, the output and affine coefficients are
propagated by copying the selected input location:
\[
    X_o^{(\ell)}
    =
    X_{q^\star}^{(\ell-1)},
\]
\[
    A_o^{(\ell)}
    =
    A_{q^\star}^{(\ell-1)},
\]
\[
    B_o^{(\ell)}
    =
    B_{q^\star}^{(\ell-1)}.
\]

\paragraph{Final fully connected layer.}
After the convolutional blocks, the feature map is flattened and passed through
the final fully connected layer. This layer is handled using the matrix
multiplication kernel inherited from the STAND-DA implementation
\citep{kiet2026statistical}. For the final fully connected layer \(\ell\) with
fixed weight matrix \(W^{(\ell)}\), the propagation is
\[
    X^{(\ell)}
    =
    X^{(\ell-1)}W^{(\ell)},
\]
\[
    A^{(\ell)}
    =
    A^{(\ell-1)}W^{(\ell)},
\]
\[
    B^{(\ell)}
    =
    B^{(\ell-1)}W^{(\ell)}.
\]
In our implementation, the final fully connected layer is bias-free. This layer
does not introduce any new selection constraint.

\section{Experiments for the extension to Deep Semi-Supervised Anomaly Detection}
\label{app:dsad-experiments}

We provide here additional experiments for the extension to Deep Semi-Supervised Anomaly Detection (Deep SAD). We compare the following methods:

- \textbf{PADI (extension to Deep SAD)}: the proposed SI method for Deep Semi-Supervised Anomaly Detection.

- \textbf{OC (extension to Deep SAD)}: 
an over-conditioning baseline that additionally conditions on the observed affine region of the frozen Deep SAD encoder, with the truncation set constructed only from that region, analogous to $\mathcal{Z}_{\mathrm{OC}}$ described in Section~\ref{subsec:characterization_Z}.

- \textbf{Naive}: traditional statistical inference.

- \textbf{Op1 (extension to Deep SAD)}: an ablation study that excludes the sign-pattern constraint.

- \textbf{Op2 (extension to Deep SAD)}: another ablation study that excludes the anomaly detection event for Deep SAD.


\subsection{Synthetic Data Experiments}

We conduct synthetic data experiments for the extension to Deep SAD under the same setting as in Section~\ref{subsec:synthetic}. The results are shown in Fig.~\ref{fig:dsad-synthetic-results}. 
Similar to the Deep SVDD experiments, PADI and OC maintain valid FPR control in both
independent and correlated covariance settings. In contrast, the Naive approach
produces inflated FPR values, showing that directly applying statistical
testing after Deep SAD-based anomaly selection leads to invalid inference due
to the selection bias. The ablation results further illustrate the importance of the two conditioning
components in PADI. Op1, which removes the sign-pattern conditioning, fails to
control the FPR because the $\ell_1$ test statistic can no longer be represented
as the linear contrast assumed in the selective inference procedure. This
prevents the construction of the correct conditional distribution for the
selective $p$-value. Op2, which ignores the Deep SAD anomaly-selection event,
also fails to maintain FPR control because the statistical test does not account
for the data-dependent mechanism used to select anomalies. Although Op2 retains
the sign-pattern conditioning, ignoring the detector selection event leaves the
selection bias uncorrected. The similar behavior of these ablation variants to
the Deep SVDD experiments confirms that both conditioning components remain
necessary when extending PADI to the Deep SAD setting. For the TPR evaluation, PADI consistently achieves higher detection power than
OC as the signal difference $\Delta$ increases. The performance gap becomes
larger for stronger anomaly signals, indicating that PADI preserves more
statistical power while accounting for the necessary selection events. These
results confirm that the proposed inference procedure can be naturally extended
from Deep SVDD to Deep SAD without sacrificing statistical validity or
detection power.

As shown in Figs. \ref{fig:synthetic-results} and \ref{fig:dsad-synthetic-results}, Op2 consistently exhibits greater FPR inflation than Op1 under both independent and correlated covariance settings. This difference can be attributed to the extent to which data-dependent selection is accounted for in each ablation. Op1 retains the anomaly-selection adjustment, whereas Op2 discards it entirely. Although Op1 remains invalid because it does not fully account for the conditional treatment of the $\ell_1$ statistic, retaining the anomaly-selection event can partially mitigate selection bias. In contrast, Op2 completely ignores the selection bias induced by the anomaly detector, leading to a greater distortion of the null distribution and, consequently, larger FPR inflation.

\begin{figure*}[tbp]
    \centering

    \begin{subfigure}[tbp]{0.49\linewidth}
        \centering
        \includegraphics[width=\linewidth]{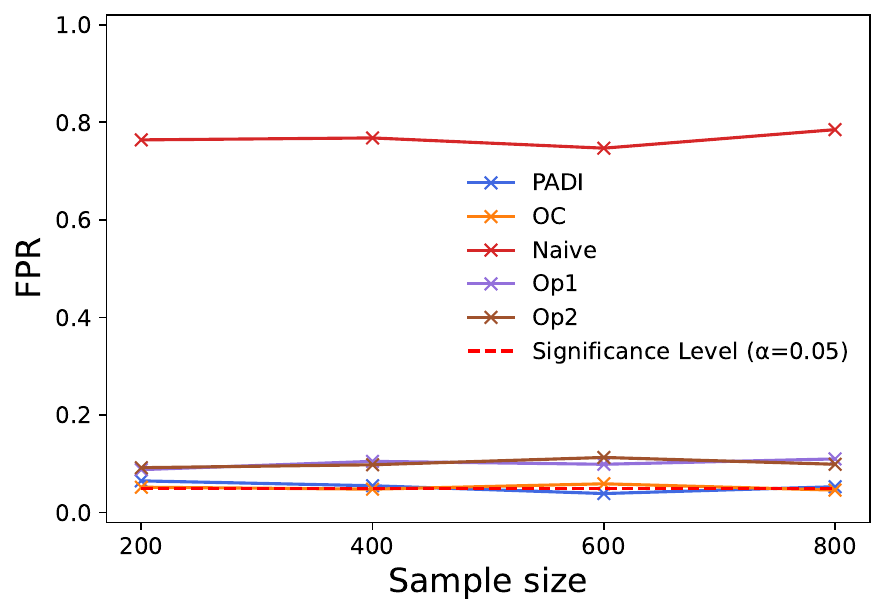}
        \caption{FPR on independent data}
        \label{fig:synthetic-fpr-independent}
    \end{subfigure}
    \hfill
    \begin{subfigure}[tbp]{0.49\linewidth}
        \centering
        \includegraphics[width=\linewidth]{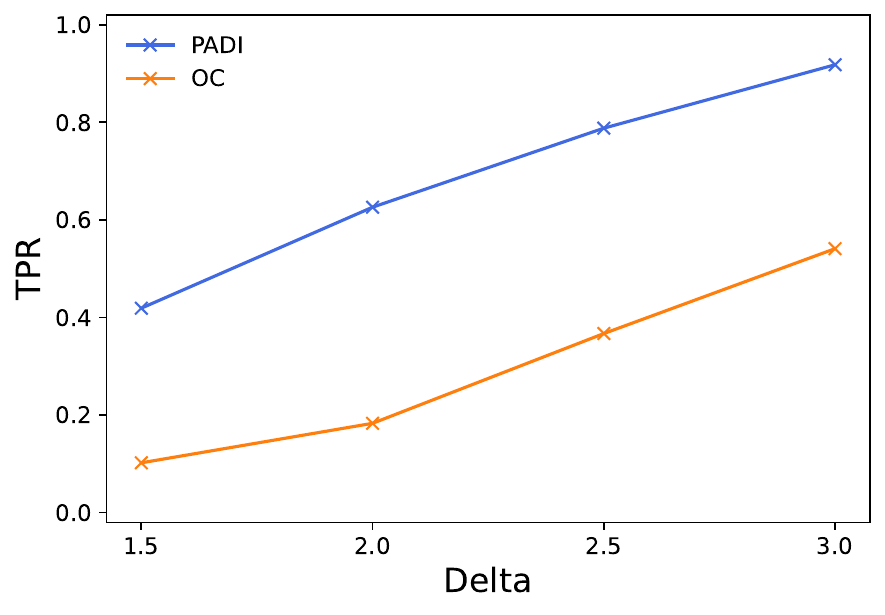}
        \caption{TPR on independent data}
        \label{fig:synthetic-tpr-independent}
    \end{subfigure}

    \vspace{0.8cm}

    \begin{subfigure}[tbp]{0.49\linewidth}
        \centering
        \includegraphics[width=\linewidth]{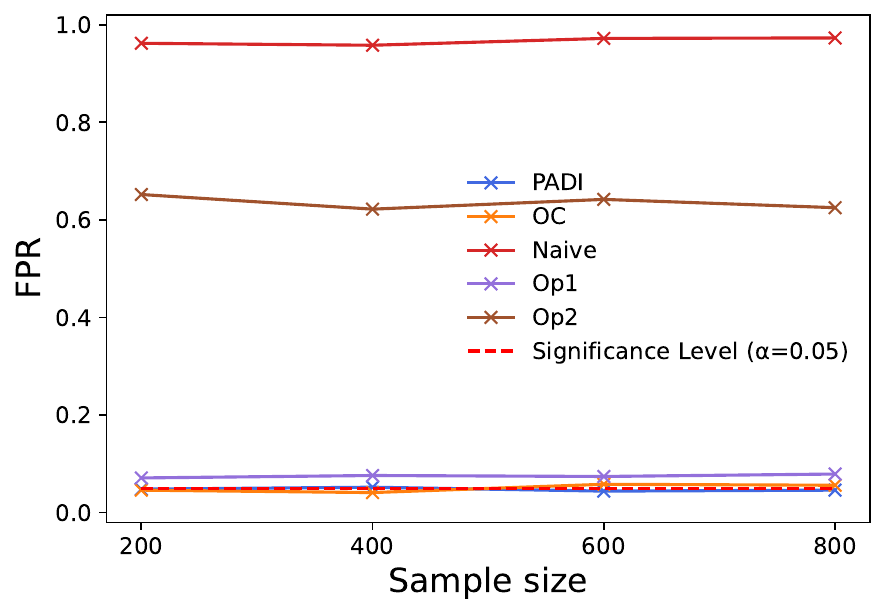}
        \caption{FPR on correlated data}
        \label{fig:synthetic-fpr-correlated}
    \end{subfigure}
    \hfill
    \begin{subfigure}[tbp]{0.49\linewidth}
        \centering
        \includegraphics[width=\linewidth]{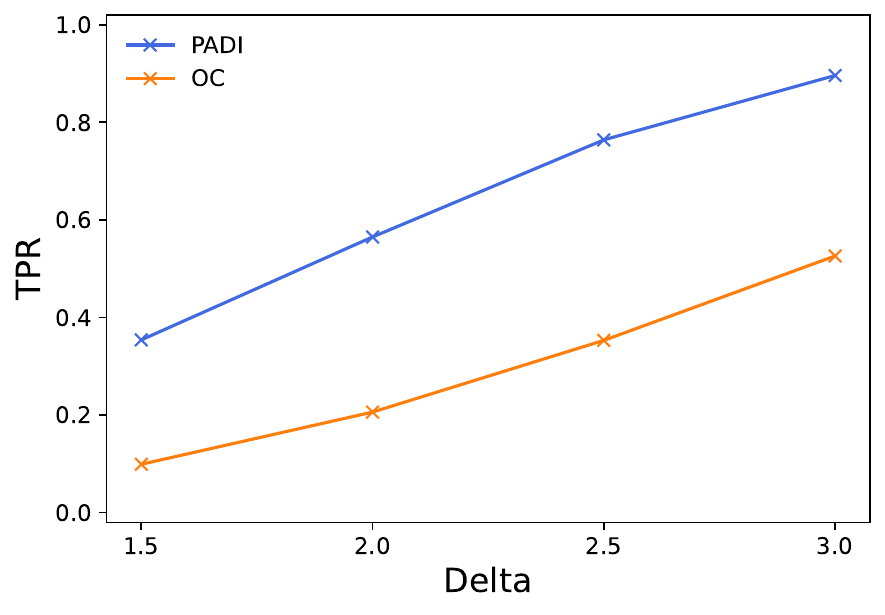}
        \caption{TPR on correlated data}
        \label{fig:synthetic-tpr-correlated}
    \end{subfigure}

    \caption{Results on synthetic data for the extension to Deep SAD. }
    \label{fig:dsad-synthetic-results}
\end{figure*}

\subsection{Real Tabular Data}

We evaluate the extension to Deep SAD on tabular datasets under the same setting as in Section~\ref{subsec:real_tabular}.
The results on real-world tabular datasets are shown in Fig.~\ref{fig:dsad-tabular-results}.
PADI and OC consistently control the empirical FPR around the desired significance
level across all evaluated datasets, whereas the Naive approach suffers from
large FPR inflation. This observation indicates that the statistical validity
provided by PADI is preserved even when the underlying anomaly detector is
trained in a semi-supervised manner.
In terms of TPR, PADI outperforms OC on all datasets. The consistent performance advantage of PADI over OC shows that the extension
to Deep SAD retains the main benefit of PADI: avoiding unnecessary conditioning
while still accounting for the data-dependent anomaly selection event.
Therefore, the proposed framework provides a reliable post-selection inference
procedure for both Deep SVDD and Deep SAD
models.

\begin{figure}[tbp]
    \centering
    \includegraphics[width=\linewidth]{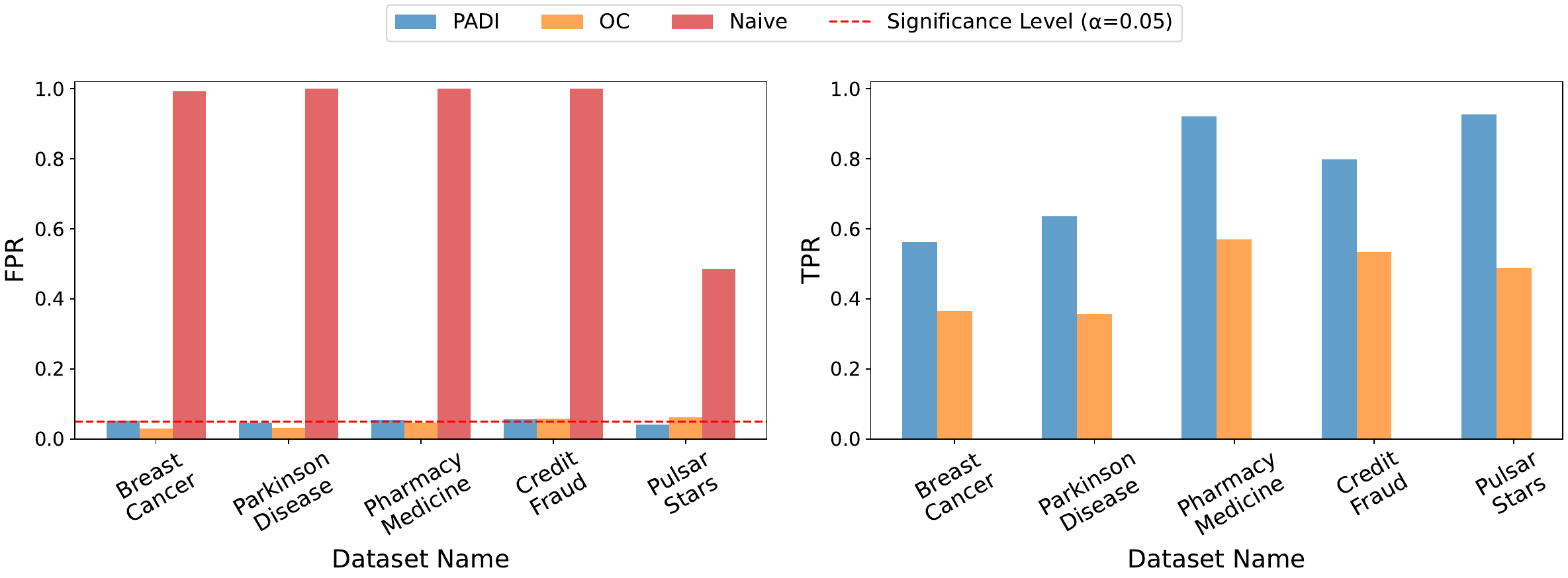}
    \caption{
    Results on real tabular datasets for the extension to Deep SAD. 
    Left: empirical FPR of PADI, OC, and Naive. 
    Right: empirical TPR of the statistically valid methods. 
    PADI consistently achieves stronger TPR than OC while preserving FPR control.
    }
    \label{fig:dsad-tabular-results}
\end{figure}

\subsection{Real Image Data}

In the real-image data experiment for the extension to Deep SAD, we used the same setting as in Section~\ref{subsec:real_image}. The results on MVTec AD are shown in Fig.~\ref{fig:dsad-image-results}. Across all
five image categories, PADI and OC maintain the empirical FPR around the target
significance level, while the Naive approach produces inflated FPR values.
This demonstrates that the selective inference framework remains valid when the
encoder is obtained from Deep SAD training. Furthermore, PADI consistently
achieves higher TPR than OC across different defect categories. The advantage of
PADI is observed across different types of image anomalies considered in the
experiments, including regular-texture, random-texture, and object categories.

\begin{figure}[tbp]
    \centering
    \includegraphics[width=\linewidth]{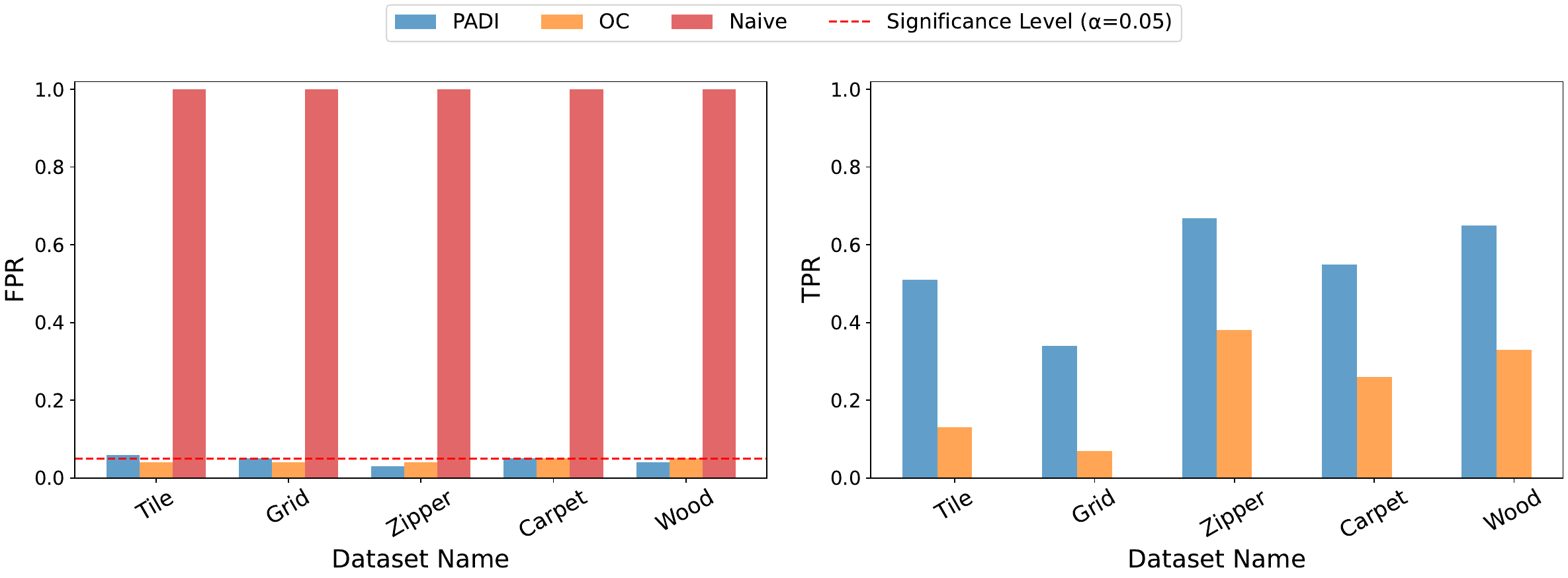}
    \caption{
    Results on real image data from the MVTec AD dataset for the extension to Deep SAD. 
    Left: empirical FPR of PADI, OC, and Naive across five classes. 
    Right: empirical TPR of the statistically valid methods. 
    PADI empirically maintains FPR close to the significance level while achieving higher TPR than OC.
    }
    \label{fig:dsad-image-results}
\end{figure}

\section{Additional Experimental Analysis}
\subsection{Robustness to Non-Gaussian Data Distributions}
\label{subsec:non_gaussian_robustness}
The theoretical guarantee of PADI is derived under the Gaussian test-reference
model, which enables the characterization of the conditional null distribution
after conditioning on the anomaly-selection event, the sign pattern, and the
nuisance statistic. To further investigate the empirical robustness of PADI
beyond the Gaussian assumption, we conduct additional synthetic experiments
using several non-Gaussian distributions.

Specifically, we consider three different non-Gaussian settings: (i) the
skew-normal distribution with a shape parameter of $0.3$, (ii) the Student's
$t$-distribution with degrees of freedom $\mathrm{df}=7$, and (iii) the Laplace
distribution with scale parameter $b=0.8$. For each distribution, the remaining
experimental settings are kept identical to the synthetic FPR experiments in
Section~\ref{subsec:synthetic}. We vary the sample size as
$n\in\{200,400,600,800\}$ and evaluate the empirical FPR at significance level
$\alpha=0.05$.

The results are shown in
Fig.~\ref{fig:non_gaussian_fpr}. Across all three non-Gaussian distributions,
PADI maintains the empirical FPR around the target significance level for all
tested sample sizes. These results indicate that, although the current
theoretical guarantee is established under Gaussian assumptions, the proposed
method can remain empirically robust under several types of non-Gaussian data
distributions considered in our experiments.

\begin{figure}[t]
    \centering
    \begin{subfigure}{0.325\linewidth}
        \centering
        \includegraphics[width=\linewidth]{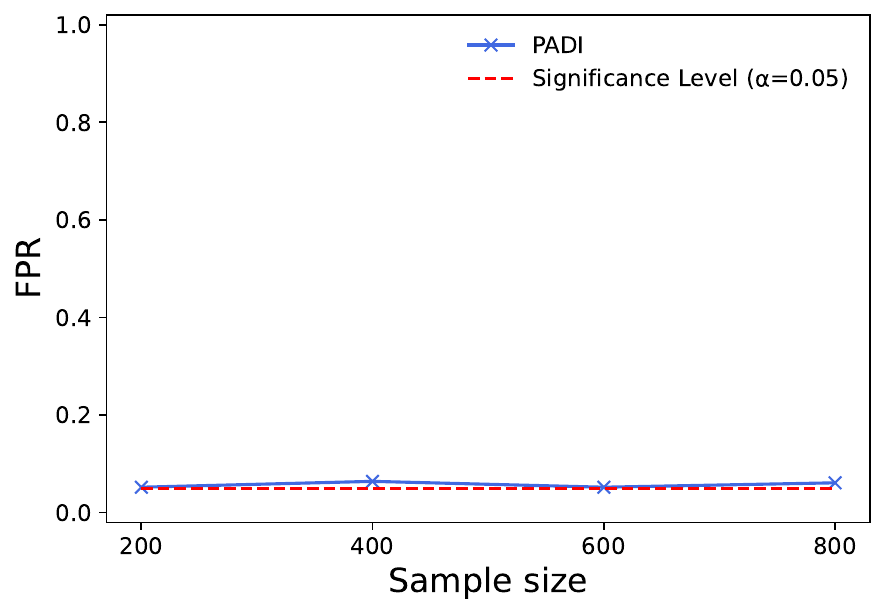}
        \caption{Skew-normal}
    \end{subfigure}
    \hfill
    \begin{subfigure}{0.325\linewidth}
        \centering
        \includegraphics[width=\linewidth]{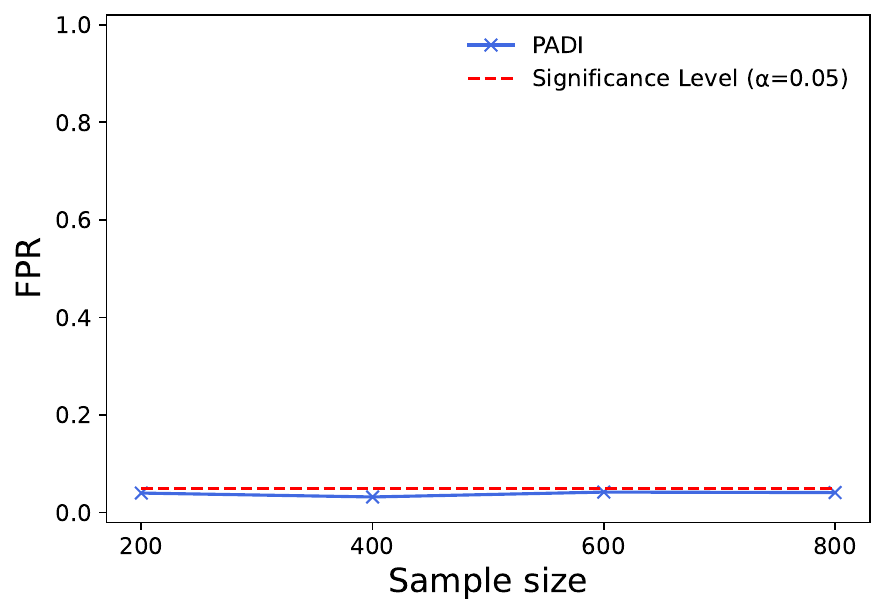}
        \caption{Student's $t$}
    \end{subfigure}
    \hfill
    \begin{subfigure}{0.325\linewidth}
        \centering
        \includegraphics[width=\linewidth]{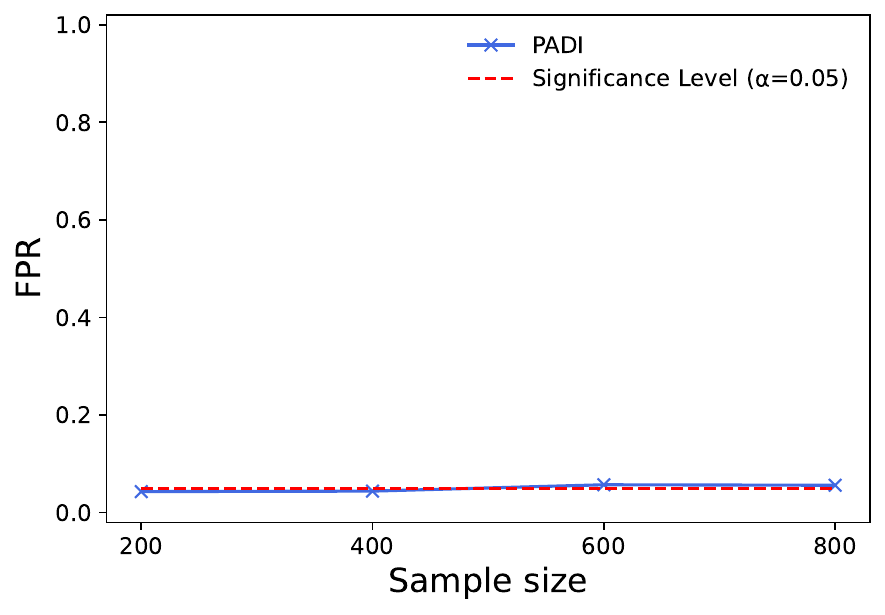}
        \caption{Laplace}
    \end{subfigure}

    \caption{
    Empirical FPR of PADI under non-Gaussian data distributions.
    }
    \label{fig:non_gaussian_fpr}
\end{figure}

\subsection{Effect of the Reference Set Size}
\label{app:reference_size}

We further investigate how the number of independent normal reference samples affects the performance of PADI. 
In particular, we vary the size of the reference set while following the same data generation procedure as the independent setting described in \ref{subsec:synthetic}. We consider reference set sizes of $5$, $10$, $15$, and $20$ samples. 

\begin{figure}[t]
    \centering
    \begin{subfigure}{0.49\linewidth}
        \centering
        \includegraphics[width=\linewidth]{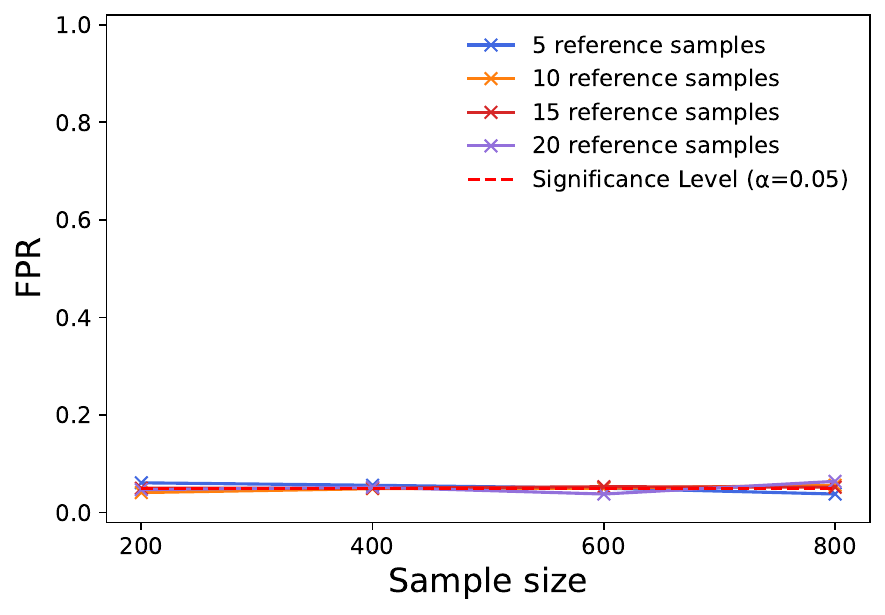}
        \caption{FPR.}
        \label{fig:fpr_ref_size}
    \end{subfigure}
    \hfill
    \begin{subfigure}{0.49\linewidth}
        \centering
        \includegraphics[width=\linewidth]{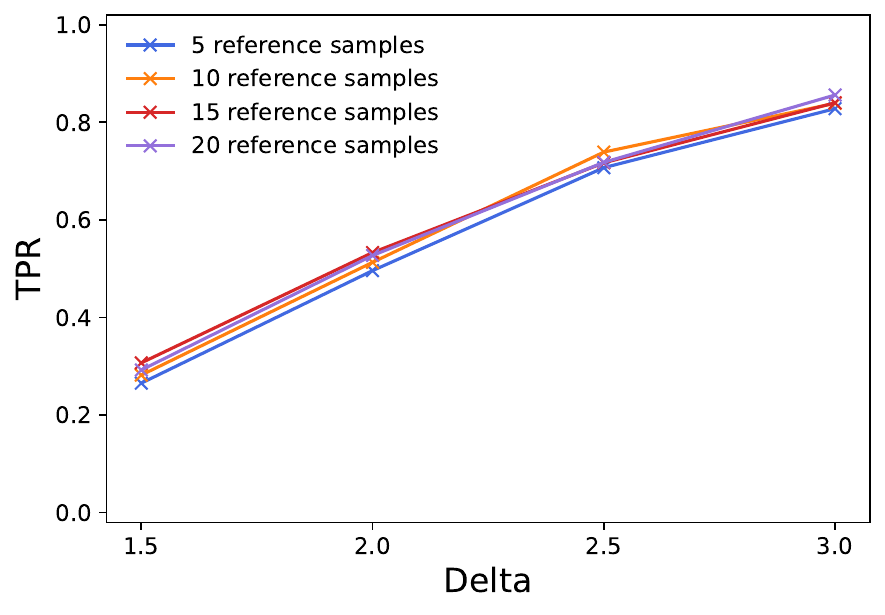}
        \caption{TPR.}
        \label{fig:tpr_ref_size}
    \end{subfigure}
    \caption{Effect of the reference set size on the empirical FPR and TPR.}
    \label{fig:reference_size_analysis}
\end{figure}

Figure~\ref{fig:fpr_ref_size} reports the empirical FPR for different numbers of reference samples.
Across different reference set sizes and test sample sizes, the empirical FPR remains consistently close to the target significance level $\alpha=0.05$. 
This confirms that the proposed selective inference procedure maintains FPR control and that the validity of PADI is not sensitive to the number of reference samples.

We also evaluate the effect of the reference set size on the TPR, as shown in Figure~\ref{fig:tpr_ref_size}. 
Although the reference set size affects the estimation of the normal reference distribution and may influence the statistical power of the test, the TPR values remain highly consistent across different reference set sizes. 
This indicates that, once a reasonable number of independent normal reference samples is available, increasing the reference set size provides limited additional improvement in detection performance.

\subsection{Effect of Data Dimension on FPR Control}
\label{app:dimension_analysis}

We further investigate whether the validity of PADI is affected by the dimensionality of the data. 
Following the same data generation procedure as the independent setting described in \ref{subsec:synthetic}, we vary the data dimension while keeping the other experimental configurations unchanged. 
Specifically, we evaluate PADI with dimensions $d \in \{20,40,60,80\}$.

\begin{figure}[t]
    \centering
    \includegraphics[width=0.6\linewidth]{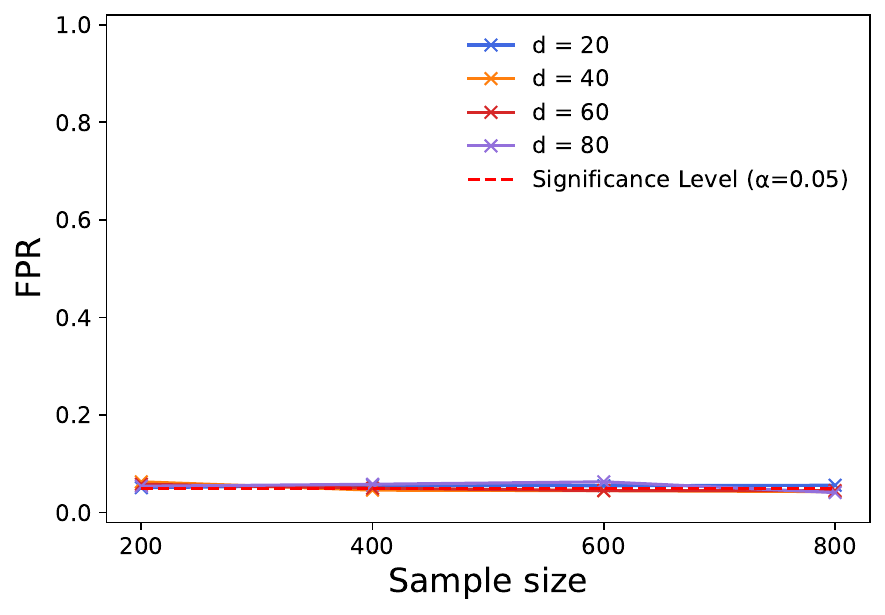}
    \caption{
    Effect of the data dimension on the empirical FPR. 
    }
    \label{fig:fpr_dimension}
\end{figure}

Figure~\ref{fig:fpr_dimension} reports the empirical false positive rate (FPR) under different data dimensions. 
Across all considered dimensions and sample sizes, the empirical FPR remains consistently close to the target significance level $\alpha=0.05$. 
This result indicates that the selective inference procedure of PADI remains valid when varying the data dimension.

\subsection{Additional Evaluation for Synthetic Experiments}
\label{subsec:additional_tpr_synthetic}

In the main synthetic experiments, we report TPR only for methods that
successfully control the FPR at the target significance level, since statistical
power comparison is meaningful only among methods that provide valid inference.
To further analyze the behavior of all competing methods, we additionally
report the TPR results of Naive, Op1, and Op2, despite their invalid FPR control.

\begin{figure*}[!t]
    \centering

    \begin{subfigure}{0.49\linewidth}
        \centering
        \includegraphics[width=.9\linewidth]{synthetic-fpr-svdd-independ_plot.pdf}
        \caption{FPR on independent data}
        \label{fig:add-synthetic-fpr-independent}
    \end{subfigure}
    \hfill
    \begin{subfigure}{0.49\linewidth}
        \centering
        \includegraphics[width=.9\linewidth]{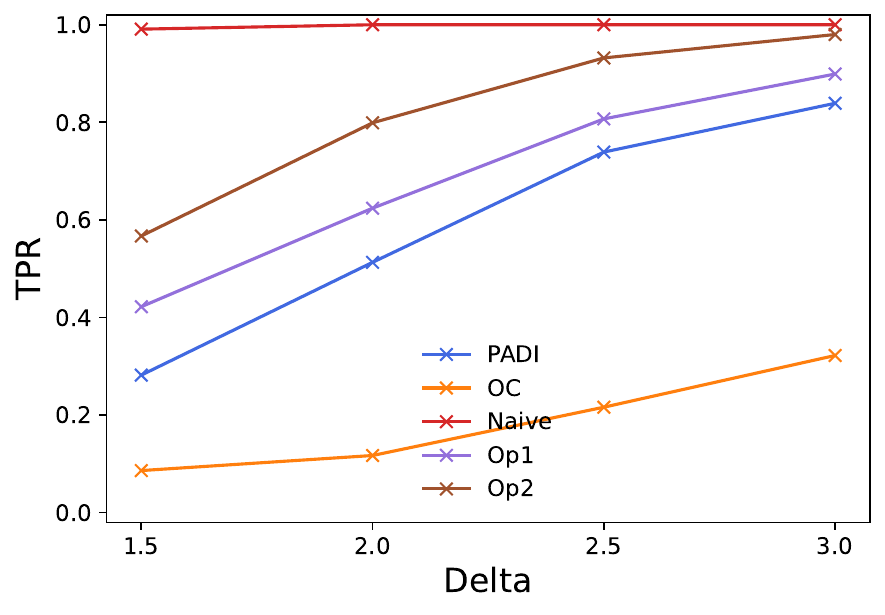}
        \caption{TPR on independent data}
        \label{fig:add-synthetic-tpr-independent}
    \end{subfigure}

    \vspace{5pt}

    \begin{subfigure}{0.49\linewidth}
        \centering
        \includegraphics[width=.9\linewidth]{synthetic-fpr-svdd-corr_plot.pdf}
        \caption{FPR on correlated data}
        \label{fig:add-synthetic-fpr-correlated}
    \end{subfigure}
    \hfill
    \begin{subfigure}{0.49\linewidth}
        \centering
        \includegraphics[width=.9\linewidth]{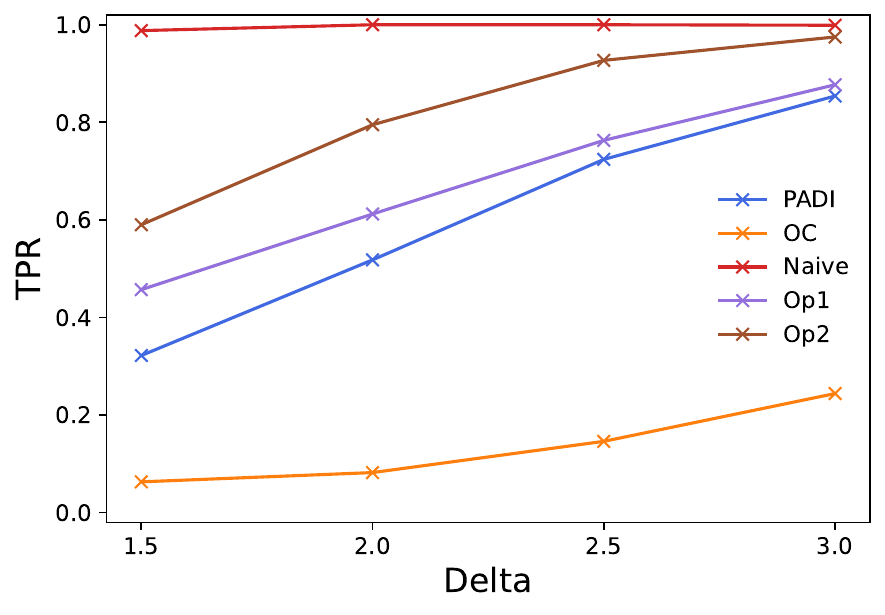}
        \caption{TPR on correlated data}
        \label{fig:add-synthetic-tpr-correlated}
    \end{subfigure}

    \caption{Additional evaluation on synthetic data for all competing methods.}
    \label{fig:add-synthetic-results}
    \vspace{-15pt}
\end{figure*}

The results are shown in
Fig.~\ref{fig:add-synthetic-results}. Similar to the main experiments, PADI and OC
maintain the empirical FPR around the target significance level in both
independent and correlated covariance settings, whereas Naive, Op1, and Op2
produce substantially inflated FPR values. As expected, Naive achieves the highest TPR
among all methods in both covariance settings, followed by Op1 and Op2. However,
these higher TPR values are accompanied by uncontrolled FPR, indicating that the
increase in detection power is obtained at the cost of invalid statistical
inference. Among the methods with valid FPR control, PADI consistently achieves
higher TPR than OC as the signal difference $\Delta$ increases. This additional
evaluation further highlights the importance of considering both statistical
validity and detection power when comparing post-selection inference methods.

\subsection{Runtime Breakdown of PADI}
\label{subsec:runtime_breakdown}

To further analyze the computational cost of PADI, we provide a runtime
breakdown of the three phases described in Algorithm~\ref{alg:padi}. The experiment follows the same setting as the ``Different numbers of
convolutional blocks'' experiment described in the ``Runtime evaluation of
CNN-based GPU kernels'' in \ref{subsec:synthetic}, with the
proposed Numba-CUDA implementation for CNN-based PADI. The runtime of each
phase is reported for CNN encoders with different numbers of convolutional
blocks.

\begin{table}[t]
\centering
\caption{
Runtime breakdown (in seconds) of the three phases in Algorithm~\ref{alg:padi}
using the proposed Numba-CUDA implementation.
}
\label{tab:runtime_breakdown}
\begin{tabular}{lcccc}
\toprule
 & 4 blocks & 5 blocks & 6 blocks & 7 blocks \\
\midrule
Phase I   & 0.01081 & 0.01102 & 0.01151 & 0.01142 \\
Phase II  & 9.46515 & 12.71300 & 24.60653 & 45.99283 \\
Phase III & 0.24853 & 0.16548 & 0.19514 & 0.15108 \\
\bottomrule
\end{tabular}
\end{table}

The results show that Phase II dominates the overall computational cost across
all tested encoder depths. This is expected because Phase II involves the
iterative line search and repeated CNN forward propagations required to
identify the feasible truncation region. In contrast, Phase I and Phase III
introduce relatively small overhead, as they mainly involve constructing the
selective inference direction and evaluating the final truncated Gaussian
probability. As the number of convolutional blocks increases, the runtime of
Phase II grows substantially, reflecting the additional cost of repeated forward
passes through deeper encoders.

\subsection{Runtime Breakdown: Training, Inference, and the Proposed PADI Method}
\label{subsec:runtime_breakdown_train_test_inference}

We conducted an additional experiment to compare the runtime of the different stages and quantify the additional computational cost of performing selective inference after a Deep SVDD model has been trained and used for standard inference. We use the same experimental setting as that described in the “Runtime Evaluation of CNN-Based GPU Kernels” in \S \ref{subsec:synthetic}, with the number of blocks fixed at four. All three stages---Deep SVDD training, standard Deep SVDD inference, and selective inference---are performed on an NVIDIA Tesla P100 GPU. For the SI stage, we use our proposed Numba-CUDA implementation for CNN-based PADI. The results are presented in Table \ref{tab:runtime_breakdown_train_test_inference}.

\begin{table}[t]
\centering
\caption{
Runtime Breakdown: Training, Inference, and the Proposed PADI Method
}
\label{tab:runtime_breakdown_train_test_inference}
\begin{tabular}{lc}
\toprule
 Stage & Runtime (s) \\
\midrule
Deep SVDD training   & 615.34 \\
Deep SVDD inference  & 0.9007 \\
PADI & 9.3698 \\
\bottomrule
\end{tabular}
\end{table}

\end{document}